\documentclass[letter,11pt]{article}
\usepackage[margin=1in]{geometry}
\usepackage{amsmath,amsthm,amsfonts,amssymb,bm}
\usepackage{algorithm,algorithmic}
\usepackage{authblk,natbib}
\usepackage{graphicx,url,footnote}
\usepackage{color}
\newtheorem{lemma}{Lemma}
\newtheorem{theorem}{Theorem}
\newtheorem{proposition}{Proposition}
\newtheorem{definition}{Definition}

\newcommand{\mc}[1]{\mathcal{#1}}
\newcommand{\mb}[1]{\mathbb{#1}}
\newcommand{\mr}[1]{\mathrm{#1}}
\newcommand{\x}{\bm{x}}
\newcommand{\z}{\bm{z}}
\newcommand{\w}{\bm{w}}
\newcommand{\X}{\bm{X}}
\newcommand{\B}{\bm{B}}

\title{Generalization Bounds of SGLD for Non-convex Learning:\\
Two Theoretical Viewpoints}
\author[1]{Wenlong Mou\thanks{mouwenlong@pku.edu.cn}}
\author[1]{Liwei Wang\thanks{wanglw@cis.pku.edu.cn}}
\author[2]{Xiyu Zhai\thanks{ayanami0@mail.ustc.edu.cn}}
\author[1]{Kai Zheng\thanks{zhengk92@pku.edu.cn}}
\affil[1]{Key Laboratory of Machine Perception, School of EECS, Peking University}
\affil[2]{School of Mathematics, University of Science and Technology of China}
\begin{document}

\maketitle

\begin{abstract}
	Algorithm-dependent generalization error bounds are central to statistical learning theory. A learning algorithm may use a large hypothesis space, but the limited number of iterations controls its model capacity and generalization error. The impacts of stochastic gradient methods on generalization error for \emph{non-convex} learning problems not only have important theoretical consequences, but are also critical to generalization errors of deep learning.
	
	In this paper, we study the generalization errors of Stochastic Gradient Langevin Dynamics (SGLD) with non-convex objectives. Two theories are proposed with non-asymptotic discrete-time analysis, using Stability and PAC-Bayesian results respectively. The stability-based theory obtains a bound of $O\left(\frac{1}{n}L\sqrt{\beta T_k}\right)$, where $L$ is uniform Lipschitz parameter, $\beta$ is inverse temperature, and $T_k$ is aggregated step sizes. For PAC-Bayesian theory, though the bound has a slower $O(1/\sqrt{n})$ rate, the contribution of each step is shown with an exponentially decaying factor by imposing $\ell^2$ regularization, and the uniform Lipschitz constant is also replaced by actual norms of gradients along trajectory. Our bounds have no implicit dependence on dimensions, norms or other capacity measures of parameter, which elegantly characterizes the phenomenon of "Fast Training Guarantees Generalization" in non-convex settings. This is the first algorithm-dependent result with reasonable dependence on aggregated step sizes for non-convex learning, and has important implications to statistical learning aspects of stochastic gradient methods in complicated models such as deep learning.
\end{abstract}

\section{Introduction}
One of the main goals of modern statistical learning theory is to derive algorithm-dependent and data-dependent generalization bounds for learning algorithms and models. A learning algorithm may use a large hypothesis space, but its randomized way of exploring the space controls actual capacity in a data-dependent manner. As a result, algorithm-dependent bounds usually go beyond classical notions of model capacities, such as VC dimensions and Rademacher complexities. For stochastic gradient methods in particular, the number of iterations and step sizes serve as implicit regularization and restrict the growth of model capacity. Algorithm-dependent generalization bounds have been intensively studied for SGM under convex settings~\citep{hardt2015train,lin2016optimal,lin2016generalization}, but very few is known for the non-convex case. Nevertheless, practitioners believe the latter to hold true in a regime far beyond existing theories. The prevailing success of stochastic gradient methods is also attributed not only to computational speed, but also to its learning-theoretic merits, known as "Train Faster, Generalize Better".

The most important arena for algorithm-dependent bound is perhaps deep learning. It is revealed by experiments that the algorithm-independent model capacities are too large to guarantee meaningful generalization performance~\citep{zhang2016understanding}. With natural images as inputs, they show that a standard neural networks can fit completely noisy labels in the training set. Obviously, such a network has no generalization power at all, and if the capacity of neural network itself was the only thing to control the generalization performance, the DNN models in real-world use would be at the same risk. Fortunately, a key difference between training procedures with random labels and true labels lies in the running time: the random labels will cost SGD algorithm significantly more steps to reach optimal point. Therefore, it is possible that good generalization performance with real labels can be guaranteed by algorithm-dependent bounds for stochastic gradient methods, while the running time for training with random labels becomes too large to yield reasonable bounds. In this sense, classical wisdom of algorithm-dependent generalization bounds could find its place critical to understanding generalization performance of deep learning, and bounds for stochastic gradient methods with non-convex objectives are central to this question.

Therefore, the goal of this paper is to understand the effect of stochastic gradient methods on generalization performance with non-convex risk minimization. We would also like to emphasize that algorithm-dependent bounds for multi-pass non-convex optimization algorithms play a much more non-trivial role than their convex counterparts: single pass of SGD for convex objectives already achieves optimality in stochastic optimization; but in non-convex settings, the computational aspects naturally requires going through training data for much more than one passes. We adopt the standard settings in learning theory, where we perform the (regularized) empirical risk minimization procedure:
\begin{equation}
\mathop{\mathrm{minimize}}_{\bm{w}}\left\{F_n(\bm{w})=\frac{1}{n}\sum_{i=1}^n f_i(\bm{w})+R(\bm{w})\right\}.
\end{equation}
We are interested in the generalization error, which is defined as the gap between empirical and population loss. We consider the error by taking expectation with respect to randomized algorithm.
\begin{equation}
\mathrm{err}(\bm{w})\triangleq\mathbb{E}_{\mathcal{A}}\left(\mathbb{E} \ell(\bm{w};z)-\hat{\mathbb{E}}_n \ell(\bm{w};z)\right)
\end{equation}
The loss function $f_i(\cdot)$ for optimization algorithm may coincide with $\ell(\cdot,z_i)$, or is a surrogate function of $\ell$ (e.g. in classification problems we usually use hinge loss as a surrogate for 0-1 loss)

Instead of working on SGD itself, we consider Stochastic Gradient Langevin Dynamics, a popular variant of stochastic gradient methods which adds isotropic Gaussian noise in each round of gradient updates, i.e.,
\begin{equation}
	\bm{w}_{k+1}=\bm{w}_k-\eta_k\hat{\bm{g}}_{k}(\bm{w})+\sqrt{\frac{2\eta_k}{\beta}} \mathcal{N}(0,I_d)
\end{equation}
We assume that the algorithm is initialized with Gaussian distribution $\bm{w}_0\sim \mathcal{N}(0,\sigma_0^2I_d)$. The stochastic gradients $\hat{\bm{g}}_k$ in each round are unbiased estimates for $\bm{\nabla}F_n(\bm{w})$, which can be decomposed as $\hat{\bm{g}}_k(\bm{w})=\bm{g}_k(\bm{w})+\bm{\nabla}R(\bm{w})$. Popular choices for $\bm{g}_k(\bm{w})$ includes full gradient $\bm{g}_k=\bm{\nabla}F_n$ and one-point stochastic gradient $\bm{g}_k=\bm{\nabla} f_{i_k}$ with $i_k\sim\mathrm{i.i.d.} \mathcal{U}\{1,2,\cdots, n\}$.

To obtain data-dependent and algorithm-dependent bounds, we adopt two theoretical tools: uniform stability~\citep{elisseeff2005stability,rakhlin2005stability} and PAC-Bayesian theory~\citep{mcallester2003pac,germain2016pac}. These two approaches not only make it convenient to analyze generalization properties along optimization trajectory, but also provide different viewpoints towards the effect of SGLD on generalization: stability only depends on relative location between parameter trained with neighboring datasets, and $O(1/n)$ fast rates are usually available; on the other hand, PAC-Bayes bounds can benefit from norm-based regularization, and it is also adaptive to optimization trajectory, instead of taking worst-case upper bounds.

The main contributions of this paper are thus two-fold. The two generalization bounds obtained by two methods reveals different aspects in which SGLD controls model complexity. It is important to note that the bounds have no dependence on dimension of parameter space, nor do they explicitly depend on norm of parameters. By assuming only the Lipschitz assumption on the objective function, the generalization bounds are controlled by aggregated step sizes. The informal versions of our results are stated as follows:
\begin{theorem}[Uniform Stability, Informal]
	Assuming $f_i(\cdot)$ is uniformly $L$-Lipschitz, let $\bm{w}_N$ be result of SGLD at $N$-th round. Under regularity conditions on the tail behavior, the following inequality holds, where the expectation in LHS is taken with respect to random draw of training data.
	\begin{equation}
		\mathbb{E}\left[\mathrm{err}(\bm{w}_N)\right]\leq O\left(\frac{1}{n}\left(k_0+L\sqrt{\beta\sum_{k=k_0+1}^{N}\eta_k}\right)\right)
	\end{equation}
	where $k_0\triangleq\min\{k: \eta_k\beta L^2<1\}$
\end{theorem}
\begin{theorem}[PAC-Bayesian Theory, Informal]
	For regularized ERM problem with $R(\bm{w})=\frac{\lambda}{2}\Vert \bm{w}\Vert^2$, let $\bm{w}_N$ be result of SGLD at $N$-th round. Under regularity conditions on the tail behavior and appropriate initialization, the following inequality holds with high probability:
	\begin{equation}
		\mathrm{err}(\bm{w}_N)\leq O\left(\sqrt{\frac{\beta}{n}\sum_{k=1}^{N} \eta_k e^{-\frac{\lambda}{2} (T_N-T_k)} \mathbb{E} \Vert \bm{g}_k\Vert^2}\right)
	\end{equation}
	where $T_k=\sum_{j=1}^k \eta_j$
\end{theorem}
The stability-based bounds exhibit a faster $O(1/n)$ rate of convergence, with complexity factor mainly depends on square root of aggregated step sizes. The PAC-Bayes bounds, though having a slower $O(1/\sqrt{n})$ rate, can make impact of step sizes in earlier iterations decay with time. The uniform Lipschitz parameter is also replaced with norm of actual gradients along optimization path. Both results greatly advance algorithm-dependent generalization bounds for non-convex stochastic gradient methods in existing literature~\citep{raginsky2017non,hardt2015train}. With the help of Gaussian noise, they even outperforms previous results in the convex case assuming constant $\beta$: the former bound allows us to perform $o(n^{\frac{2}{1-\alpha}})$ gradient updates for step sizes $\eta_k=ck^{-\alpha}$. In the second bound, generalization error is controlled not only by what step sizes parameter we take, but also how large the actual steps are. In most optimization problems including deep learning, the norm of gradient diminishes along trajectory, as the iteration approaches a stationary point, even if uniform Lipschitz constants are very large. This phenomenon pushes above PAC-Bayesian bounds into a favorable situation, where the earlier large gradient steps are greatly abated by the exponentially decaying factor, while the gradients taken in latter stage are inherently small.

\subsection{Related Work}
The effect of stochastic gradient methods on statistical learning has attracted lots of interests in existing literature: For linear regression in Hilbert spaces,~\citet{lin2016optimal,lin2016generalization} analyze multi-pass stochastic gradient methods, leading to optimal population risks; More general cases are studied via uniform stability of parameters under $\ell_2$ norm~\citep{hardt2015train,london2016generalization}; From statistical inference aspects,~\citet{chen2016statistical} constructed confidence sets based on the Markov chain induced by SGD for strongly-convex objective functions. Most of them requires objective function to be convex. While~\citet{hardt2015train} considered non-convex smooth objective functions, their results require $O(1/k)$ fast decay of step sizes, and the bound has exponential dependence on smoothness parameter. With the presence of Gaussian noise, our bounds for non-convex objectives become even better than their results in convex case.

Deliberate injection of Gaussian noise has become a rising star in the literature of non-convex optimization.~\citet{ge2015escaping,jin2017escape} show that Gaussian noise helps SGD escape 2nd order saddle points efficiently. Stochastic Gradient Langevin Dynamics, proposed as discrete version of Langevin Equation $dw_t=-\bm{\nabla} F(\bm{w}_t)dt+\sqrt{\frac{2}{\beta}} d\bm{B}_t$, also plays an important role in optimization and sampling. It is well-known that Langevin Equation asymptotically converges to equilibrium distribution $p(\bm{w})\propto e^{-\beta F(\bm{w})}$, see e.g.~\citep{markowich2000trend}. This property has been utilized for posterior sampling, known as Langevin Monte Carlo. The discretization error and mixing time are intensively studied by~\citet{bubeck2015sampling,nagapetyan2017true}, for log-concave distributions.~\citet{dalalyan2012sparse} also used Langevin MC to approximate Exponential Weighted Aggregate, and proved PAC-Bayesian bounds for regression learning with sparsity prior. For non-convex learning and optimization,~\cite{raginsky2017non} makes the first attempt towards excess risks by non-convex SGLD, combining algorithmic convergence and generalization error. But their results are based on convergence to equilibrium, which relies upon constants in Poincar\'e Inequality, leading to inevitably exponential dependence on dimension. Though the mixing time can be prohibitive in non-convex case,~\citet{zhang2017hitting} recently show that hitting time of SGLD for small-loss region can be much better, and the Gaussian noise in SGLD helps to escape shallow local minima. Their results also emphasize the importance of generalization guarantees for discrete-time non-asymptotic SGLD in non-convex settings.

Besides, several recent works also studied the connection between SGD and stochastic differential equations, such as SME~\citep{li2015dynamics,li2017batch}. Though our results for SGLD cannot directly extend to their SDEs with data-dependent diffusion term, our methods are potentially applicable for generalization error bounds in their settings.

\subsection{Why Gaussian Noise is Useful for Generalization?}
Previous analyses of the Gaussian noise in stochastic gradient methods mainly focus on its benefit for optimization aspect. The question naturally comes whether it also helps generalization a lot. Before going into our theoretical results, we will first illustrate why prior analyses on stability can be very large for non-convex objective function, and how this can be overcome by adding Gaussian noise. This important observation motivates our analysis based on KL-Divergence and Hellinger distances, which highlights the effect of smooth distributions on generalization error bounds.

Stability-based analysis for gradient algorithms on non-convex losses will suffer from a "fence-sitting" situation, as illustrated in Figure~\ref{illustration}. Consider a non-convex empirical loss surface with two local minima, which is divided into two regions by a ridge. If $\bm{w}_k$ lies on one side of this ridge, a noiseless first-order method will lead to the local minimum on this side. However, if $\bm{w}_k$ comes close to the ridge in its trajectory, small shift on the loss surface caused by changing one point will lead it to a completely different local minimum, as we can see from the figure.
\begin{figure}[htb]
\begin{minipage}[t]{0.5\linewidth}
\centering
\includegraphics[width=0.8\linewidth]{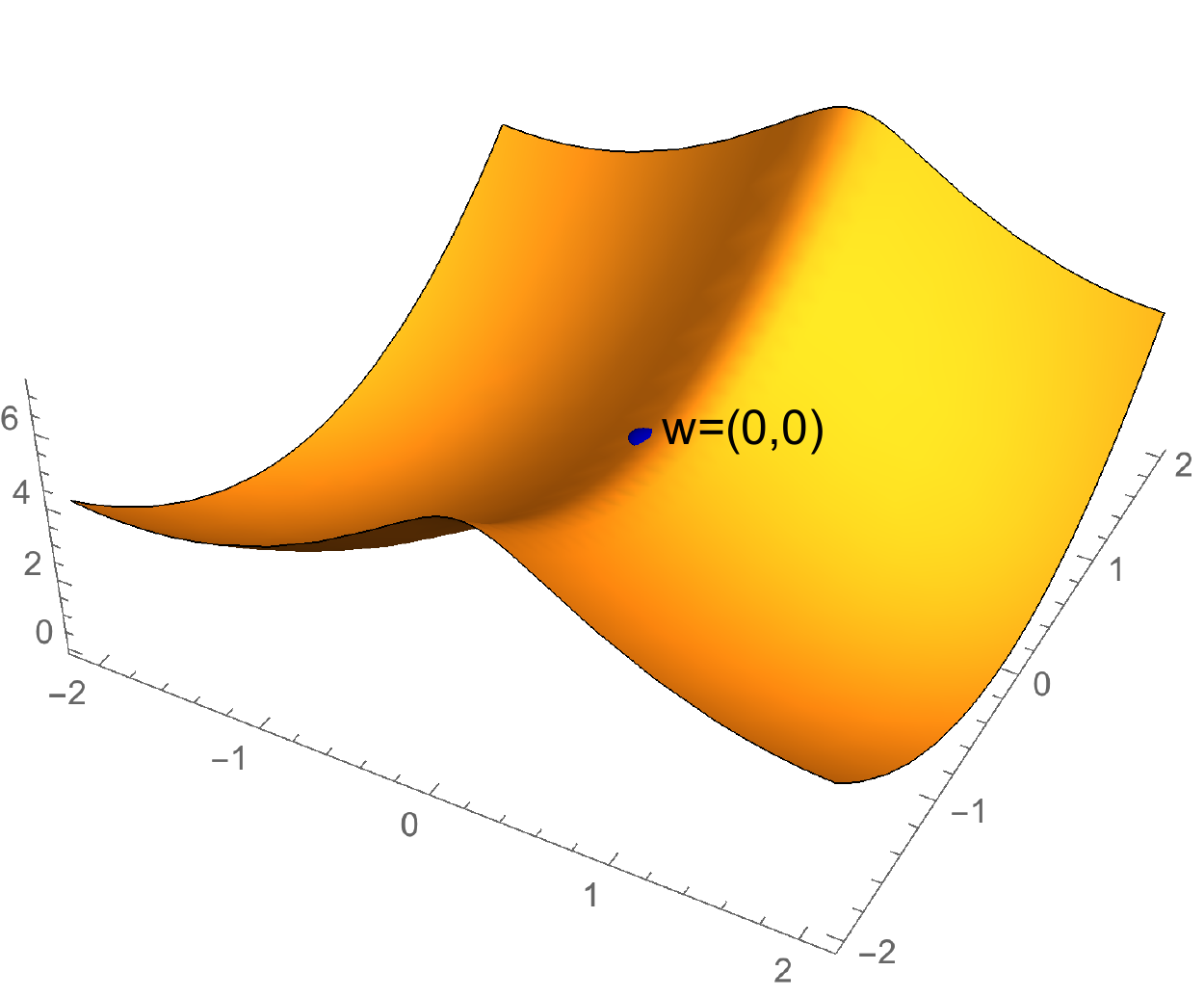}
\end{minipage}
\begin{minipage}[t]{0.5\linewidth}
\centering
\includegraphics[width=0.8\linewidth]{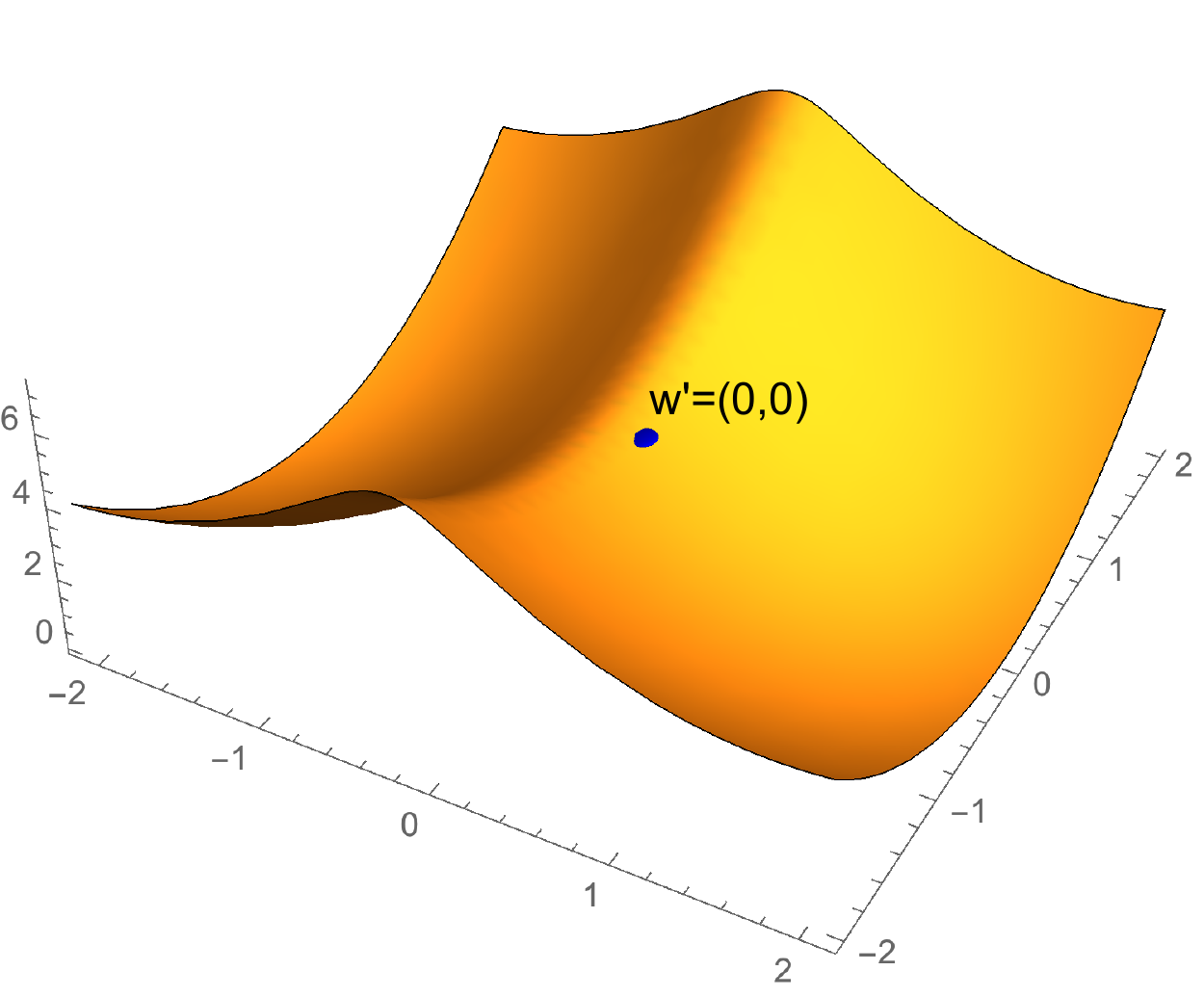}
\end{minipage}
\caption{Illustration of "Fence-Sitting" Situation for Stability of Non-convex Optimization}\label{illustration}
\end{figure}

To guarantee stability, we need $\bm{w}_k$ to randomly decide which side to go when it comes close to the ridge. The noise needs to be isotropic and smooth enough in order to cross this ridge, as the direction of variation can be quite arbitrary. SGLD successfully tackles the fence-sitting problem by smoothing the probability of going either side, and adding noise to subsequent steps to avoid unstable shallow local minima. The bounds for SGD in~\citet{hardt2015train} also exploits randomness of choosing $i_k$, but the noise is not smooth enough. So their bound requires the subsequent steps to be very small, to keep $w_k$ not far from the ridge.

\section{Preliminaries}
\textbf{Notation:} Suppose each data $z_i \in \mc{Z} (\forall i \in \{1,2,\dots, n\})$. A pair of neighboring datasets $S, S' \in \mathcal{Z}^n$ means that $S$ and $S'$ differ on exactly one data point. For a continuous time SDE over $S$, the iteration point at time $t$ is denoted as $\w_t$, and corresponding density function is denoted as $\pi_t(\w)$. For discrete time SGLD run over $S$, the iteration point and its density function at round $k$ are written as $\w_k, p_k(\w)$ respectively. All above notations are also suitable for $S'$ with an additional prime. When analyzing their derivatives, we sometimes omit the subscript $t$ for $\pi_t,\pi_t'$ without confusion. $\eta_k$ is the step size of discrete SGLD at iteration $k$, and $T_k \triangleq \sum_{j=1}^k \eta_j$. Let $\bm{g}_k(\cdot)$ be the stochastic gradient operator at round $k$ without regularization, and let $\hat{\bm{g}}_k(\bm{w})=\bm{g}_k(\bm{w})+\bm{\nabla}R(\bm{w})$ be the actual stochastic gradient. Without extra explanations, $L$ represents the Lipschitz constant of the objective function $f(\cdot;\z)$ for any $\z$. $D_H(p||q)$ represents the squared Hellinger distance between density function $p$ and $q$. 

\begin{equation}
	D_H(p||q)\triangleq \frac{1}{2}\int_{\mathbb{R}^d} \left(\sqrt{p}-\sqrt{q}\right)^2dw
\end{equation}

Now we define an important concept which will be frequently used later:
\begin{definition}[non-expansive]
 Suppose $\w$ and $\w'$ are two random points in $\mb{R}^d$, and their distributions are denoted as $\mc{P}(\w), \mc{P}(\w')$. We say a bivariate functional $D(\cdot || \cdot)$ defined on two density functions, is non-expansive, if for any mapping $\psi: \mb{R}^d \rightarrow \mb{R}^d$, there is
 \begin{equation}
   D(\mc{P}(\psi(\w)) || \mc{P}(\psi(\w')) \leqslant D(\mc{P}(\w) || \mc{P}(\w'))
  \end{equation} 
\end{definition}
It is well known that all $f$-divergence (including KL divergence and squared Hellinger distance) are non-expansive and jointly convex~\citep{csiszar2004information}.

\subsection{Stability and generalization}
Stability of the algorithm has a close relation with its generalization performance, and this line of research dates back to \cite{bousquet2002stability}. Intuitively, the more stable an algorithm is, the better its generalization performance will be. Here, we adopt the notion of uniform stability of a randomized algorithm \cite{elisseeff2005stability,hardt2015train}.

\begin{definition}[Uniform Stability]
  We say that a randomized algorithm $A$ is $\epsilon_n$-uniformly stable with respect to the loss $\ell$, if for all neighboring datasets $S, S' \in \mathcal{Z}^n$, there is 
  \begin{equation*}
   \sup_{\bm{z}} |\mb{E}[\ell(\bm{w}_S; z)] - \mb{E}[\ell(\bm{w}_{S'}; z)]| \leqslant \epsilon
  \end{equation*}
  where the expectation is over randomness of the algorithm, and $\w_S, \w_{S'}$ are outputs of $A$ on $S$ and $S'$ respectively. 
\end{definition}  

Once a randomized algorithm is uniformly stable, it is straightforward to see its generalization performance in expectation, using standard symmetrization argument~\citep{hardt2015train}.
\begin{theorem}[Generalization in expectation]
 Suppose a randomized algorithm $A$ is $\epsilon$-uniformly stable, then there is 
 \begin{equation*}
  |\mb{E}[\mr{err}(\w_S)]| \leqslant \epsilon
 \end{equation*}
\end{theorem}
High-probability bounds with an additional $O(\sqrt{\frac{\log 1/\delta}{n}})$ term are also available by assuming uniformly bounded loss~\citep{elisseeff2005stability}. In this paper, we always take expectation with respect to randomized learning algorithm when discussing generalization bounds. Under suitable assumptions, it is straightforward to extend our results to high-probability guarantees with respect to random draw of training data, using McDiarmid Inequality. For simplicity of presentation, we restrict our attention to $\epsilon$ itself and expected generalization bounds.

\subsection{PAC-Bayesian theory}
Different with uniform stability theory above, which needs to consider the worst case in some sense, the generalization bound implied by PAC-Bayesian theory is completely algorithmic and data dependent. However, most of generalization bounds in PAC-Bayesian form require the loss function to be bounded \citep{mcallester1999pac,mcallester2003pac}, which is usually not satisfied in reality, such as cross entropy loss or hinge loss. \cite{germain2016pac} extended previous results to sub-Gaussian losses, but their result introduced an extra additive error term $\frac{1}{2}s^2$, where $s^2$ is the Sub-Gaussian variance factor. To get rid of this additive term and facilitate our later analysis, we first improve the PAC-Bayesian result in \cite{germain2016pac} as follows: 

\begin{theorem}
For loss function class $\{f(w;x)\}$ and data distribution $x\sim \mathcal{D}$. Given any prior distribution $\mathcal{P}$ over $\Omega$. If loss class is $s$-subGaussian with respect to $\mathcal{D}\times \mathcal{P}$, i.e.,
\begin{equation}
\mathbb{E} e^{\lambda f(w;x)}\leq e^{\frac{1}{2}\lambda^2s^2},\quad \forall \lambda>0,
\end{equation}
Let $\Xi$ be a class of posterior distributions over $\Omega$, with $\sup_{\mathcal{Q}\in\Xi} D_{KL}(Q||P)\leq M$, we have the following inequality holds uniformly for all posterior distributions $\mathcal{Q}\in \Xi$, with probability $1-\delta$:
\begin{equation}
 \mathbb{E}_{\mathcal{D}}\mathbb{E}_{\mathcal{Q}} f(w;x)\leq \hat{\mathbb{E}}_n \mathbb{E}_{\mathcal{Q}}f(w;x) +O\left(s\sqrt{\frac{D_{KL}(\mathcal{Q}||\mathcal{P})\vee 1+\log\frac{1}{\delta}+\log\log M}{n}}\right)
\end{equation}
\end{theorem} 

\subsection{Fokker-Planck equation}
As we know, the movement of a particle in the $d$-dimensional space influenced by its current state and random forces (here we only consider a simple case), can be characterized by the following stochastic differential equation (SDE):
\begin{equation}
 d\X_t = \bm{\mu}(\X_t, t)dt + \sqrt{2 \beta^{-1}} d\B_t
\end{equation}
where $\X_t$ is the random position of the particle at time $t$, $\bm{\mu}(\X_t, t)$ is the $d$-dimensional random drift vector, and $B_t$ is the $d$ dimensional Brownian motion. Denote the density function of $\X_t$ as $p(\x,t)$, then Fokker-Planck equation describes the evolution of $p(\x,t)$:
\begin{equation}
 \frac{\partial p(\x,t)}{\partial t} = \frac{1}{\beta}\Delta p(\x, t) - \bm{\nabla}\cdot(p(\x,t)\bm{\mu}(\x, t))
\end{equation}
where $\Delta$ is the Laplace operator.
\section{Ideal Case: Generalization Bounds for Langevin Equation}

Intuitively, SGLD can be seen as a discretization for Langevin Equation. Understanding generalization performance of the ideal continuous-time algorithm provides important insights into deep results about discrete-time algorithm. In this section, we will present two generalization error bounds for SGLD, using stability and PAC-Bayesian theory, respectively. We elaborate on the techniques used in our analysis, which gives a high-level view of how generalization bound for discrete-time SGLD can be obtained.

Consider the following continuous-time Langevin Equation, where $F_n$ is (regularized) empirical objective function.
\begin{equation}
	d\bm{w}(t)=-\bm{\nabla} F_n(\bm{w}(t))dt+\sqrt{2 \beta^{-1}}d\bm{B}(t),\quad t\ge 0
\end{equation}
where $\{B(t)\}_{t\ge 0}$ is the standard Brownian motion in $\mathbb{R}^d$.

Assume the pdf of $\bm{w}(t)$ is $\pi_t(\bm{w})$, then it satisfies a Fokker-Planck equation:
\begin{equation}
	\frac{\partial \pi}{\partial t}=\frac{1}{\beta}\Delta \pi+\bm{\nabla}\cdot(\pi\bm{\nabla} F_n)
\end{equation}

\subsection{Uniform Stability}

We are going to bound uniform stability with respect to loss function, which directly controls generalization in expectation:
\begin{equation}
	\epsilon_n=\sup_{z,|S\Delta S'|=1}\left\{\Big|\mathbb{E}_{\mathcal{A}}\ell(\mathcal{A}(S);z)-\mathbb{E}_{\mathcal{A}}\ell(\mathcal{A}(S');z)\Big|\right\}
\end{equation}

For uniform stability, we assume that $f(\bm{w};\bm{z})$ the following condition which is slightly weaker than uniform Lipschitz. Note that the generalization performance is defined in terms of loss function $\ell$, which may not be continuous, but the Lipschitz assumption is imposed on objective $f$ of our algorithm, which can be a surrogate function for $\ell$.
\begin{equation}
	\forall z,z',\quad\|\nabla f(\bm{w};z)-\nabla f(\bm{w};z')\|\le L
\end{equation}

As a result, we have for different samples $S,S'$, $|S \Delta S'|=1$,
\begin{equation}
	\|\bm{\nabla}F_n-\bm{\nabla}F_n'\|\le \frac{L}{n}
\end{equation}

We first control $\epsilon_n$ via squared Hellinger distance:
\begin{equation}\label{stability-to-hellinger}
\begin{split}
\epsilon_n=&\sup_{x,S,S'}\left|\int_{\mathbb{R}^d}\ell(\bm{w};z)\pi_t(\bm{w})dw-\int_{\mathbb{R}^d}\ell(\bm{w};z)\pi_t(\bm{w})dw\right|\\
	=&\sup_{x,S,S'}\left|\int_{\mathbb{R}^d}\ell(\bm{w};z)\left(\sqrt{\pi}+\sqrt{\pi'}\right)\left(\sqrt{\pi}-\sqrt{\pi'}\right)dw\right|\\
	\leq&\sup\left\{\left(\int_{\mathbb{R}^d}\ell(\bm{w};z)^2\left(\sqrt{\pi}+\sqrt{\pi'}\right)^2dw\right)^{\frac{1}{2}}\left(\int_{\mathbb{R}^d}\left(\sqrt{\pi}-\sqrt{\pi'}\right)^2dw\right)^{\frac{1}{2}}\right\}\\
	=&2\sup_{\pi_t} \Vert \ell\Vert_{L^2(\pi_t)}\sqrt{D_{H}(\pi||\pi')}\\
	\leq& 2C\sqrt{D_{H}(\pi||\pi')}\\
	\end{split}
\end{equation}
The last inequality holds by assuming $\ell$ is uniformly bounded by $C$.

Compared with~\cite{hardt2015train}, the bound based on $f$-divergence can better characterize stability with non-convex objective: through one step of iteration, the $L^2$ distance between parameters $\mathbb{E}\Vert w_k-w_k'\Vert^2$ can expand a lot due to shape of non-convex surface, but $f$-divergences are non-expansive under the same transformation, and will decrease by convolution with Gaussian noise. This property makes it possible to obtain much better bounds.

\begin{proposition}
	Under above assumptions, the expected generalization error for continuous-time Langevin Equation is bounded by:
	\begin{equation}
		\mathbb{E} [\mathrm{err}(\bm{w}_T)]\leq \frac{LC\sqrt{\beta T}}{\sqrt{2}n}
	\end{equation}
\end{proposition}
\begin{proof}
According to the analysis above, we only need to bound $D_H(\pi||\pi')$ from above.

Apparently, at time $t=0$, $D_H(\pi||\pi')=0$. We then estimate $\frac{d}{dt}D_{H}(\pi_t||\pi_t')$:
\begin{equation}
\begin{split}
\frac{d}{dt}D_{H}(\pi_t||\pi_t')=&-\int_{\mathbb{R}^d} \frac{\partial }{\partial t}\sqrt{\pi \pi'}dw\\
&=-\int_{\mathbb{R}^d}\frac{\sqrt{\pi'}}{2\sqrt{\pi}}\frac{\partial \pi}{\partial t}dw-\int_{\mathbb{R}^d}\frac{\sqrt{\pi}}{2\sqrt{\pi'}}\frac{\partial \pi'}{\partial t}dw\\
&=-\int_{\mathbb{R}^d}\frac{\sqrt{\pi'}}{2\sqrt{\pi}}\left(\frac{1}{\beta}\Delta \pi+\bm{\nabla}\cdot(\pi \bm{\nabla}F_n)\right)dw-\int_{\mathbb{R}^d}\frac{\sqrt{\pi}}{2\sqrt{\pi'}}\left(\frac{1}{\beta}\Delta \pi'+\bm{\nabla}\cdot(\pi' \bm{\nabla}F_n')\right)dw\\
&=\frac{1}{2}\int_{\mathbb{R}^d}\bm{\nabla}\frac{\sqrt{\pi'}}{\sqrt{\pi}}\cdot\left(\frac{1}{\beta}\nabla\pi+\pi \bm{\nabla}F_n\right)dw+\frac{1}{2}\int_{\mathbb{R}^d}\bm{\nabla}\frac{\sqrt{\pi}}{\sqrt{\pi'}}\left(\frac{1}{\beta}\bm{\nabla}\pi'+\pi' \bm{\nabla}F_n')\right)dw\\
\end{split}
\end{equation}
The last equality is due to integration by parts. Technical conditions such as uniform decaying tails of $\pi$ and $\pi'$ can be found in~\citep{risken1989fokker}. We then proceed to calculate the part induced by gradient update (with coefficient $1$) and those induced by Gaussian convolution (with coefficient $\frac{1}{\beta}$) individually, which can be described as follows:
\begin{equation}
\begin{split}
\frac{d}{dt}D_{H}(\pi_t||\pi_t')&=\frac{1}{2}\int_{\mathbb{R}^d}\bm{\nabla}\frac{\sqrt{\pi'}}{\sqrt{\pi}}\cdot\left(\frac{1}{\beta}\nabla\pi+\pi \bm{\nabla}F_n\right)dw+\frac{1}{2}\int_{\mathbb{R}^d}\bm{\nabla}\frac{\sqrt{\pi}}{\sqrt{\pi'}}\left(\frac{1}{\beta}\bm{\nabla}\pi'+\pi' \bm{\nabla}F_n')\right)dw\\
&=\frac{1}{4}\int_{\mathbb{R}^d}\sqrt{\pi \pi'}\bm{\nabla}\log\frac{\pi'}{\pi}\cdot\left(\frac{1}{\beta}\bm{\nabla}\log\pi+ \bm{\nabla}F_n\right)dw+\frac{1}{4}\int_{\mathbb{R}^d}\sqrt{\pi \pi'}\bm{\nabla}\log\frac{\pi}{\pi'}\cdot\left(\frac{1}{\beta}\bm{\nabla}\log\pi'+\bm{\nabla}F_n')\right)dw\\
&=-\frac{1}{4}\int_{\mathbb{R}^d}\sqrt{\pi \pi'}\left(\frac{1}{\beta}\|\bm{\nabla}\log\frac{\pi'}{\pi}\|^2+\bm{\nabla}\log\frac{\pi}{\pi'}\cdot(\bm{\nabla}F_n-\bm{\nabla}F_n')\right)dw\\
&\le\frac{1}{4}\int_{\mathbb{R}^d}\frac{\beta}{2}\sqrt{\pi \pi'}\|\bm{\nabla}F_n-\bm{\nabla}F_n'\|^2dw\\
&\le \frac{\beta L^2}{8n^2}
\end{split}
\end{equation}
Integrating through time and plugging into the estimate above, we have:
\begin{equation}
	\epsilon_n\leq 2C \sqrt{D_H(\pi_T||\pi'_T)}\leq \frac{LC\sqrt{\beta T}}{\sqrt{2}n}
\end{equation}
\end{proof}
\subsection{PAC-Bayesian Bounds}
We can also obtain PAC-Bayesian bounds for Fokker-Planck Equation with finite $T$. 

\begin{proposition}\label{ideal-pac-bayes-1}
	Let prior distribution $\gamma=\mathcal{N}(0,\sigma_0^2I)$. Assume that $\ell(w;x)$ is $s$-subGaussian with respect to $\gamma\times \mathcal{D}$. Within a class of posteriors with uniform upper bound $D_{KL}(\pi||\gamma)\leq M$, the following holds for Langevin Dynamics with probability $1-\delta$:
	\begin{equation}
		\mathrm{err}(\bm{w}_T)\leq  s\left( \frac{\beta}{2n}\int_{0}^T e^{\frac{-(T-t)}{2\beta \sigma_0^2}}\mathbb{E}_{\pi_t} \left\Vert \bm{\nabla}F_n+\frac{1}{\beta}\bm{\nabla}\log \gamma\right\Vert^2dt+\frac{\log 1/\delta+\log\log M}{n}\right)^{\frac{1}{2}}
	\end{equation}
\end{proposition}
\begin{proof}
We only need to bound the KL divergence to prior distribution $\gamma$.
\begin{equation}
\begin{split}
\frac{d}{dt}D_{KL}(\pi_t||\gamma)=&\int_{\mathbb{R}^d} \frac{\partial \pi}{\partial t}(\log \pi+1-\log \gamma)dw\\
=&-\frac{1}{\beta}\int_{\mathbb{R}^d} \pi\Vert \bm{\nabla} \log \pi-\bm{\nabla}\log\gamma\Vert^2dw+\int_{\mathbb{R}^d} \pi\langle \bm{\nabla}F_n+\frac{1}{\beta}\bm{\nabla} \log \gamma, \bm{\nabla} \log \pi-\bm{\nabla} \log \gamma\rangle dw\\
\leq &-\left(\frac{1}{\beta}-\frac{1}{2C}\right)\int_{\mathbb{R}^d} \pi\Vert \bm{\nabla} \log \pi-\bm{\nabla}\log\gamma\Vert^2dw+\frac{C}{2}\int_{\mathbb{R}^d} \pi \Vert \bm{\nabla}F_n+\frac{1}{\beta}\bm{\nabla}\log \gamma\Vert^2dw
\end{split}
\end{equation}
We use Cauchy-Schwartz inequality in the second step, and the constant $C$ will be determined later.
The first term is minus Fisher information $I(\pi||\gamma)$, which can be upper bounded by $-D_{KL}(\pi||\gamma)$ itself using logarithmic Sobolev inequality~\citep{markowich2000trend}:
\begin{equation}
D_{KL}(\pi||\gamma)\leq \sigma_0^2I(\pi||\gamma),\quad\text{for } \gamma=\mathcal{N}(0,\sigma_0^2I)
\end{equation}
Let $C=\beta$ and plug into the log Sobolev inequality, we get:
\begin{equation}
	\frac{d}{dt}D_{KL}(\pi_t||\gamma)\leq -\frac{1}{2\beta\sigma_0^2}D_{KL}(\pi_t||\gamma)+\frac{\beta}{2}\int_{\mathbb{R}^d} \pi_t \Vert \bm{\nabla}F_n+\frac{1}{\beta}\bm{\nabla}\log \gamma\Vert^2dw
\end{equation}
Solving for $D_{KL}$ with initial value $D_{KL}(\pi_0||\gamma)=0$, we get:
\begin{equation}
	D_{KL}(\pi_T||\gamma)\leq \frac{\beta}{2}\int_{0}^T e^{\frac{-(T-t)}{2\beta \sigma_0^2}}\mathbb{E}_{\pi_t} \left\Vert \bm{\nabla}F_n+\frac{1}{\beta}\bm{\nabla}\log \gamma\right\Vert^2dt
\end{equation}
Since we use Gaussian prior, the second term in the expectation can be directly calculated as $\frac{1}{\beta}\bm{\nabla}\log \gamma=-\frac{1}{\beta\sigma_0^2}\bm{w}$, making the bound dependent on $\ell_2$ norm of the parameter. This is undesirable in the high-dimensional settings: as $w_0\sim \mathcal{N}(0,I_d)$, $w_0$ concentrates around $\sigma_0^2d$ with high probability, resulting in a term linearly dependent on $d$. Fortunately, this can be eliminated by imposing a small $\ell_2$ regularization term.
\end{proof}

Instead of minimizing empirical risk itself, we consider the regularized ERM problem with regularization term $R(\bm{w})=\frac{\lambda}{2}\Vert \bm{w}\Vert^2$. To make the gradient of $R(\cdot)$ cancel out with the $\bm{\nabla}\log \gamma$ term, we choose $\lambda=\frac{1}{\beta \sigma_0^2}$. Using the same method of analysis, we get:
\begin{equation}
	\frac{d}{dt}D_{KL}(\pi_t||\gamma)\leq -\frac{1}{2\beta\sigma_0^2}D_{KL}(\pi_t||\gamma)+\frac{\beta}{2}\int_{\mathbb{R}^d} \pi_t \Vert \bm{\nabla}\hat{\mathbb{E}}_n f+\lambda\bm{w}+\frac{1}{\beta}\bm{\nabla}\log \gamma\Vert^2dw
\end{equation}
Using the same methods as before, we get:
\begin{proposition}\label{ideal-pac-bayes-2}
	Under the same assumptions as in Proposition~\ref{ideal-pac-bayes-1}, the Langevin Equation for regularized ERM problem with $\lambda=\frac{1}{\beta\sigma_0^2}$ satisfies:

\begin{equation}
	\mathrm{err}(\bm{w})\leq s\left(\frac{\beta}{2n}\int_{0}^T e^{-\frac{\lambda}{2}(T-t)}\mathbb{E}_{\pi_t} \left\Vert \bm{\nabla}\hat{\mathbb{E}}_n f(\bm{w})\right\Vert^2dt+\frac{\log 1/\delta+\log\log M}{n}\right)^{\frac{1}{2}}
\end{equation}
\end{proposition}

By assuming uniform $L$-Lipschitzness of $f_i$, we can get a simpler upper bound:
\begin{equation}
	\mathrm{err}(\bm{w})\leq sL\sqrt{\frac{\beta(1-e^{-\frac{ \lambda T}{2}})}{\lambda n}}+O\left(\frac{1}{\sqrt{n}}\right)
\end{equation}

\section{Stability of Discrete-Time SGLD}
Though the ideal continuous-time Langevin Equation attains small generalization error, they cannot imply bounds for discrete-time SGLD algorithms. Most previous analyses relate discrete-time analysis and continuous-time ones by estimation of discretization gap, which usually results in at least linear dependence on $d$~\citep{raginsky2017non}. In our analyses, we directly construct various SDEs that are similar to Langevin Equation, based on discrete-time updates. This technique makes it possible to circumvent the potentially large gaps between discrete and continuous time algorithms, as we can see from this and next section.

In this section, we will consider the stability of SGLD algorithm for non-convex objectives. To begin with, we will give the stability result of Langevin Monte Carlo (LMC), a special case of SGLD which uses full gradient in each iteration. LMC is closer to continuous-time algorithm, relatively easy to analyze and reaches the uniform stability of $O\left(\frac{L\sqrt{\beta\sum \eta_k}}{n}\right)$. However, from the practical view, SGLD with a randomly drawn example in each round is much more attractive. Hence we extend our methods and provide analyses for SGLD algorithms. A simple analysis is first presented with stability bound of $O\left(L\sqrt{\frac{\beta\sum \eta_k}{n}}\right)$. When step sizes are small, a lot more decrease in squared Hellinger distance can be acquired, and the bound can be improved to $O\left(\frac{L\sqrt{\beta\sum \eta_k}}{n}\right)$. We also obtain a rough estimate for larger step sizes. Combining two results together, a stability bound that nearly matches the ideal case is obtained for SGLD.


\subsection{Stability of Langevin Monte Carlo}
We consider the following LMC algorithm, which uses full gradients in each update.

\begin{equation}
	\bm{w}_{k+1}=\bm{w}_k-\frac{\eta_k}{n}\sum_{i=1}^n\bm{\nabla} f(\bm{w}_k;z_i)+\sqrt{\frac{2\eta_k}{\beta}} \mathcal{N}(0,I_d)
\end{equation}

To give an intuitive analysis, suppose two neighboring datasets $S,S'$ differing only in the $i_*$-th data. Then one can divide each iteration into two parts: the first part just update $\w_k$ and $\w_k'$ with gradients over $n-1$ same data and $z_{i_*}$, i.e. 
\begin{equation}
	\w_k^{(1)} := \w_k - \frac{\eta_k}{n}\sum_{i \neq i_*} \bm{\nabla} f(\bm{w}_k;z_i)- \frac{\eta_k}{n}\bm{\nabla} f(\bm{w}_k;z_{i_*}) 
\end{equation},
\begin{equation}
 	\w_k^{(1)\prime} := \w_k' - \frac{\eta_k}{n}\sum_{i \neq i_*} \bm{\nabla} f(\bm{w}_k';z_i)- \frac{\eta_k}{n}\bm{\nabla} f(\bm{w}_k';z_{i_*}) 
 \end{equation} and then we obtain $\bm{w}_{k+1}$ and $\bm{w}_{k+1}'$ by adding Gaussian noise and replacing the gradient of sample $z_{i_*}$ in $\bm{w}_k^{(1)\prime}$ by the gradient of sample $z_{i_*}'$, i.e. $\w_{k+1}' = {\w_k^{(1)\prime}} - \frac{\eta_k}{n} \bm{\nabla} (f(\bm{w}_k;z_{i_*}')-f(\bm{w}_k;z_{i_*})) + \sqrt{\frac{2\eta_k}{\beta}} \mathcal{N}(0,I_d)$.
In the first step, squared Hellinger distance does not increase because of the non-expansive property. For the second step, one can view them as consecutive SDEs with drift term $\bm{g},\bm{g}'$ of order $O(\frac{1}{n})$. Hence we can prove the increments of $D_H(\pi||\pi')$ after one iteration is of order $O(\frac{1}{n^2})$, which leads to the following generalization bound.

\begin{theorem}[Generalization Error of LMC]
 	Assuming $\forall z,z',\forall\bm{w},\|\bm{\nabla}f(\bm{w};z)-\bm{\nabla}f(\bm{w};z')\|\le L$.

 	Let $\bm{w}_N$ be result of LMC at $N$-th round. Under regularity conditions on the tail behavior, then the following inequality holds:
	\begin{equation}
		\mathbb{E}[\mathrm{err}(\bm{w}_T)]\leq O\left(\frac{L\sqrt{\beta\sum_{k=1}^N \eta_k}}{n}\right)
	\end{equation}
	where the expectation is taken over the randomness of training data.
\end{theorem} 

\subsection{Stability of SGLD - A Succinct Analysis}
As random draw of a training example is more popular in practice, it is desirable to analyze generalization properties of SGLD. In the rest part of this section, we will assume $\bm{g}_k=\bm{\nabla} f_{i_k}(\bm{w})$, where $i_k$ is the index of randomly drawn training example. We will first present a simple analysis for stability of SGLD. Though the resulting bound is not optimal, the analysis illustrates important principles for understanding how SGLD helps stability. In the following, we will derive upper bounds for $\delta_k\triangleq D_H(p_k||p_k')$ recursively. There are two possible cases for $i_k$:
\begin{itemize}
	\item If $i_k\neq i_*$, then SGLD implemented over $S$ or $S'$ will use the same gradient mapping, i.e. $\psi_k: \bm{w}\mapsto \bm{w}-\eta_k\nabla f(\bm{w};z_{i_k})$, then we have 
	\begin{equation}
	D_{H}(\mathcal{P}(\psi_k (\bm{w}_k)|i_k)||\mathcal{P}(\psi_k (\bm{w}_k')|i_k))\leq D_{H}(p_k||p_k')=\delta_k
	\end{equation}
	Furthermore let $\mathcal{G}_k=\mathcal{N}(0,\frac{\eta_k}{\beta}I_d)$, by the convexity of squared Hellinger distance (which is implied by joint convexity of $f$-divergence), there is 
	\begin{align*}
	D_H(\mathcal{P}(\bm{w}_{k+1}|i_k)||\mathcal{P}(\bm{w}'_{k+1}|i_k))&=D_H(\mathcal{G}_k*\mathcal{P}(\psi_k (\bm{w}_k)|i_k)||\mathcal{G}_k*\mathcal{P}(\psi_k (\bm{w}_k')|i_k))\\
	&\le D_{H}(\mathcal{P}(\psi_k (\bm{w}_k)|i_k)||\mathcal{P}(\psi_k (\bm{w}_k')|i_k))\\
	&\le \delta_k
	\end{align*}
	So in this case, the SGLD update is non-expansive with respect to $\delta_k$.

	\item If $i_k=i_*$, we have nothing but limited step size in hand. The increase of $f$-divergence can be bounded through norm-based shifts in parameter space only under smoothness conditions, which is helped by Gaussian noise. Therefore, we expand the discrete-time update into a stochastic process, where the effect of gradient flow is smoothed by Gaussian at each time $t$.
\end{itemize}

Concretely, for $i_k=i_*$, the update can be interpolated as:
\begin{equation}
\forall t\in[0,\eta_k],\quad\bm{\theta}_t=\bm{\theta}_0-\int_0^t \bm{\nabla} f_{i_k}(\bm{\theta}_0)ds+\sqrt{\frac{1}{\beta}}\int_0^td\bm{B}_s,\quad \bm{\theta}_0=\bm{w}_k
\end{equation}
However, $\bm{\theta}_t$ is not a Markov process, as it always involves the initial random point $\bm{\theta}_0$. Using the same technique as in \cite{raginsky2017non}, we define $\bm{g}_t(\bm{v})\triangleq\mathbb{E}\left(\bm{\nabla} f_{i_k}(\bm{\theta}_0)\Big| \bm{\theta}_t=\bm{v}\right)$. The mimicking distribution results~\citep{gyongy1986mimicking} guarantees solution to the following SDE has the same one-time marginal as $\bm{\theta}_t$.
\begin{equation}
\begin{split}
	d\bm{v}_t = \bm{g}_s(\bm{v}_s)ds+\sqrt{\frac{2}{\beta}}d\bm{B}_s,\quad v_0\sim p_k
	\end{split}
\end{equation}
The corresponding Fokker-Planck equation for above process is:
\begin{equation}
	\frac{\partial \pi}{\partial t}=\bm{\nabla}\cdot\left(\frac{1}{\beta} \bm{\nabla}\pi+\pi \bm{g}_t\right)
\end{equation}
We also have counterparts for the neighboring dataset, denoted as $\pi_t'$. With the help of these PDEs, we can bound the variation of squared Hellinger distance.

As in the ideal case, we can compute that
\begin{equation}
\begin{split}
\frac{d}{dt}D_{H}(\pi_t||\pi_t')&=-\frac{1}{4}\int_{\mathbb{R}^d}\sqrt{\pi \pi'}\left(\frac{1}{\beta}\|\bm{\nabla}\log\frac{\pi'}{\pi}\|^2+\bm{\nabla}\log\frac{\pi}{\pi'}\cdot(\bm{g}_t-\bm{g}'_t)\right)dw\\
&\le\frac{\beta}{8}\int_{\mathbb{R}^d}\sqrt{\pi \pi'}\|\bm{g}_t-\bm{g}_t'\|^2dw\\
&\le \frac{\beta L^2}{8}
\end{split}
\end{equation}

    For $i_k=i_*$, we have 
	\begin{equation}
	D_H(\mathcal{P}(\bm{w}_{k+1}|i_*)||\mathcal{P}(\bm{w}'_{k+1}|i_*))\leqslant \delta_k+\frac{\beta L^2}{8}\eta_k
	\end{equation}

Combining above two cases and using the convexity of squared Hellinger distance, we obtain
\begin{equation}
	\delta_{k+1}\le \frac{n-1}{n}\delta_k+ \frac{1}{n}(\delta_k+ \frac{\beta L^2}{8} \eta_k)=\delta_k+\frac{\beta L^2}{8n}
\end{equation}

Putting them together, we get following guarantees for SGLD:

\begin{theorem}
	Consider $N$ rounds of SGLD with parameters $\beta$ and $\{\eta_i\}$. If we assume
	\begin{enumerate}
		\item the loss function $l(\bm{w};z)$ is uniformly bounded by $C$;
	\item $\forall z,z'$, the gradients of objective function satisfy $\|\bm{\nabla}f(\bm{w};z)-\bm{\nabla}f(\bm{w};z')\|\le L$
	\end{enumerate}
	Then we have the following generalization bound in expectation
	\begin{equation}
	\label{preliminary}
	\mathbb{E}[\mathrm{err}(\bm{w}_N)]\leq LC\left(\frac{\beta}{8n}\sum_{i=1}^k\eta_i\right)^{1/2}
	\end{equation}
\end{theorem}

\subsection{Stability of SGLD - An Improved Analysis}
Though above analysis for the stability of SGLD is intuitive, the result is not satisfactory, as the bound has a $O\left(\frac{1}{\sqrt{n}}\right)$ gap compared with Langevin MC. Technically, if we choose $i_k\neq i_*$ in $k$-th round, both $p_k$ and $p_k'$ will be smoothed by the Gaussian noise, and their squared Hellinger distance will decrease by a quadratic information-type term. This term was completely ignored in the succinct analysis, and by making use of this term we can also obtain $O(1/n)$ fast rate for SGLD.

Before proceeding into improved bound, we will first introduce a framework for combining different stability results. This is motivated by time-varying step sizes in SGLD: as step size changes, the best method of estimation may be different. To utilize their respective advantages, we first prove the following theorem.

\begin{theorem}
	\label{mixing theorem}
  	Suppose there are two types of bivariant-functionals $D_A(\cdot||\cdot)$ and $D_B(\cdot||\cdot)$ between p.d.fs for estimating stability of SGLD and there are constants $A_f$ and $B_f$ depending only on $f$ such that $\epsilon_n$ can be bounded by
  	\begin{equation}
  		\epsilon_n\le A_f D_A(p_N||p_N'),\quad\epsilon_n\le B_f D_B(p_N||p_N')
  	\end{equation}

  	For a SGLD algorithm with step sizes $\eta_1,\cdots, \eta_N$, assume we can estimate $D_A$ and $D_B$ by
  	\begin{equation}
  		D_A(p_N||p_N')\le h_A(\eta_1,\cdots, \eta_N),\quad D_B(p_N||p_N')\le h_B(\eta_1,\cdots, \eta_N)
  	\end{equation}

  	Moreover, assume $D_A$ and $D_B$ are nonexpansive and convex, then for any integer $1\le k\le N-1$, we can bound stability by
  	\begin{equation}
  		\epsilon_n  \le A_f h_A(\eta_1,\cdots, \eta_k)+ B_fh_B(\eta_{k+1},\cdots, \eta_N)
  	\end{equation}
\end{theorem}  
\begin{proof}
	We assume there is a mixed process $\mathcal{A}''$ that use samples $S$ for the first $k$ steps and samples $S'$ for the rest steps. We denote the corresponding paramters and p.d.fs by $\bm{w}_k''$ and $p_k''$.

	\begin{align*}
	\epsilon_n&=\sup\limits_{z}\Big|\int f(\bm{w};z)(p_N'(\bm{w})-p_N(\bm{w}'))\Big|\\
	&\le \sup\limits_{z}\Big|\int f(\bm{w};z)(p_N'(\bm{w})-p_N''(\bm{w}))\Big|+ \sup\limits_{z}\Big|\int f(\bm{w};z)(p_N(\bm{w})-p_N''(\bm{w}))\Big|\\
	&\le A_f D_A(p'_N||p_N'')+ B_fD_B(p_N||p_N'')\\
	\end{align*}

	Here by nonexpansiveness and that for step $k+1,\cdots,N$ the mixed process uses sample set $S$', $D_A(p'_N||p_N'')\le D_A(p_k'||p_k'')\le h_A(\eta_1,\cdots,\eta_k)$.

	Note that $p_l=p_l''$ for $l=1,\cdots, k$, then $D_B(p_N||p_N'')\le h_B(\eta_{k+1},\cdots, \eta_N)$.

	Therefore, we obtain
	\begin{equation}
		\epsilon_n  \le A_f h_A(\eta_1,\cdots, \eta_k)+ B_fh_B(\eta_{k+1},\cdots, \eta_N)
	\end{equation}
	
\end{proof}

When the step sizes are large, e.g., $\eta_k\beta L^2=\Omega(1)$, this step will make a contribution larger than $1/n$ in the succinct bound. However, a stochastic gradient step can change a distribution within at most $O(1/n)$ scale with respect to $L^1$ distance. So if step sizes are large, a rough estimate based on $L^1$ distance will be better.

First, it is easy to see that stability can be well-controlled by $L^1$ distance for bounded loss. 
\begin{equation}
 	\epsilon_n = \sup_z\left|\int l(\bm{w};z)(p_N- p_N')d w\right|\leqslant \sup\|l\|_{L^\infty}\int|p_N - p_N'|dw
\end{equation} 
Note that $L_1$ distance is also a kind of $f$-divergence. Hence with probability $1- \frac{1}{n}$, SGLD will select the same data point, which does not increase the $L_1$ distance. With probability $\frac{1}{n}$, SGLD will choose the different data point. Though it may not be easy to calculate the difference, we know the upper bound of $L_1$ distance is at most 2. Combining these together, one can prove the following stability result：

\begin{lemma}

	\label{Stability-absolute difference}
	For SGLD algorithm runs $k_0$ iterations, there is 
	\begin{equation}
		\int_{\mathbb{R}^d}|p_{k_0}(\bm{w})-p_{k_0}'(\bm{w})|dw\leqslant \frac{2k_0}{n}
	\end{equation}
\end{lemma}

With this lemma and the Theorem~\ref{mixing theorem} in hand, we are able to focus on smaller step sizes. In this situation the process is similar to a continuous-time Langevin dynamics, thus the bound for squared Hellinger distance should also be similar to that of a continuous process, which is $O\left(\frac{L \sqrt{\beta T}}{n}\right)$. By aligning the information-type term for $i_k\neq i_*$ with that of $i_k=i_*$, we can improve the bound to the order $O(\frac{1}{n})$, as in the following lemma:

\begin{lemma}
	\label{Hellinger - improved}
	 Suppose for $\forall k, \eta_k \leqslant \frac{\ln 2}{\beta L^2}$, then there is
	\begin{equation}
	\sqrt{D_H(p_N||p_N')}\le  \frac{\sqrt{\beta}L}{2n}\left(\sum\limits_k \eta_k\right)^{1/2}
	\end{equation}
\end{lemma}

\begin{proof}[Proof (Sketch)]
It is easy to see the $k$-th update of SGLD is equivalent to the following step:
\begin{align*}
\bm{w}_{k+1}=\bm{w}_k-(1-X)\eta_k\bm{\nabla} f(\bm{w}_k;z_{j_k})-X \eta_k \bm{\nabla}f(\bm{w}_k;z_{i_*})+\mathcal{N}(0,\frac{2\eta_k}{\beta}I_d),\\
\quad\quad\quad\quad \bm{w}_k\sim p_k, \quad j_k\sim\mathcal{U}(\{1,2,\cdots,n\}\setminus \{i_*\})
\end{align*}
where $\bm{w}_k,j_k,X$ are independent and $\mathcal{P}(X=1)=\frac{1}{n}, \mathcal{P}(X=0)=\frac{n-1}{n}$.

Now, we consider a family of random variables $\bm{\theta}_t~(0\le t\le \eta_k)$ defined by
\begin{equation}
	\bm{\theta}_t=\bm{w}_k-\eta_k\bm{\nabla} f(\bm{w}_k;z_{j_k})-X t (\bm{\nabla}f(\bm{w}_k;z_{i_*})-\bm{\nabla} f(\bm{w}_k;z_{j_k}))+\mathcal{N}(0,\frac{2t}{\beta}I_d)
\end{equation}

Denote the p.d.f of $\bm{\theta}_t$ as $\pi(\x, t)$. Similarly, we also define $\bm{\theta}_t', \pi'(\x,t)$. Then one can check $\pi(\x, t)$ satisfies the following PDE with $\hat{\bm{g}}(\bm{w})=\mathbb{E}_{X,\bm{w}_k,j_k}[X(\bm{\nabla}f(\bm{w}_k;z_{i_*})-\bm{\nabla} f(\bm{w}_k;z_{j_k}))|\bm{\theta}_{t}=\bm{w}]$:
\begin{equation}
	\frac{\partial \pi}{\partial t}= \frac{1}{ \beta}\triangle \pi+\bm{\nabla}\cdot(\pi \hat{\bm{g}})
\end{equation}

Similarly, we also have $\hat{\bm{g}}'$ for $\pi'$.

Although $\hat{\bm{g}}-\hat{\bm{g}}'$ is not necessarily pointwise bounded by $O(\frac{1}{n})$ as in the case of Langevin MC and continuous Langevin dynamics, we can prove a bound of order $O(\frac{1}{n})$ w.r.t to weighted average:
\begin{equation}
	\int \sqrt{\pi \pi'}\|\hat{\bm{g}}-\hat{\bm{g}}'\|^2\le \frac{4\sqrt{2}L^2}{(n-1)^2}
\end{equation}

Then as in previous analysis, we compute the time derivative of squared Hellinger distance:
\begin{align*}
\frac{d}{dt}D_H(\pi_t||\pi_t')&=-\frac{1}{4}\int_{\mathbb{R}^d}\sqrt{\pi \pi'}\left(\frac{1}{\beta}\|\bm{\nabla}\log\frac{\pi'}{\pi}\|^2+\bm{\nabla}\log\frac{\pi}{\pi'}\cdot(\hat{\bm{g}}_t-\hat{\bm{g}}'_t)\right)dw\\
&\le \frac{\beta}{8}\int \sqrt{\pi \pi'}\|\hat{\bm{g}}-\hat{\bm{g}}'\|^2\\
&<\frac{\beta L^2}{n^2}
\end{align*}

So we have
\begin{equation}
	D_H(p_{k+1}||p_{k+1}')= D_H(\pi_{\eta_k}||\pi_{\eta_k}')\le D_H(\pi_0||\pi_0')+\frac{\beta L^2}{n^2}\eta_k\le D_H(p_k||p_k')+\frac{\beta L^2}{n^2}\eta_k
\end{equation}

Then one arrives at the statement by induction.
\end{proof}

Thus, by setting $k_0:= \min\{k: \eta_k\beta L^2<\ln 2\}$, and using Theorem \ref{mixing theorem} with stability bounds, we can prove the desired result:
\begin{theorem}
	\label{stability of SGLD - final resulte}
	Consider $N$ rounds of SGLD with parameters $\beta$ and $\{\eta_i\}$. Suppose the loss function $l(\bm{w};z)$ is uniformly bounded by $C$, and $\forall z,z'$, there is $\|\bm{\nabla}f(\bm{w};z)-\bm{\nabla}f(\bm{w};z')\|\le L$. By setting $k_0$ such that $\eta_{k_0} \leqslant \frac{\ln 2}{\beta L^2}$, then we have the following generalization bound in expectation
	\begin{equation}
	\mathbb{E}[\text{err}(\bm{w}_N)]\leq \frac{2k_0}{n} +\frac{\sqrt{\beta}LC}{n}\left(\sum\limits_{i=k_0+1}^N \eta_i\right)^{1/2}
	\end{equation}
\end{theorem}

\section{PAC-Bayesian Theory for Discrete-Time SGLD}
In this section, we present a non-asymptotic analysis for the generalization performance of SGLD using PAC-Bayesian theory. As in previous section, we directly construct stochastic processes and corresponding PDEs based on the discrete-time update, instead of estimating the discretization gap. We add $\ell_2$ regularization term $R(\bm{w})=\frac{\lambda}{2}\Vert \bm{w}\Vert^2$ to the ERM objective, in order to avoid norm-dependent term in generalization bound. However, the uniform way of interpolating the stochastic process in previous section will lead the effect of regularization to conditional expectation term $\mathbb{E}\left[\bm{\theta}_0|\bm{\theta}_t=\bm{w}\right]$, which cannot cancel perfectly with $\bm{w}$. Therefore, we construct the stochastic process and corresponding PDE in a non-uniform way. We also allow prior $\gamma_k$ to vary with $k$ in a data-independent way in order to match with the regularization term. Since full gradients and stochastic gradients play the same role in the PAC-Bayesian analysis for SGLD, the choice of $\bm{g}_k$ can be arbitrary, and we use the abstract notation of $\bm{g}_k$.
\subsection{Constructing the PDEs}
The following theorem relates discrete-time updates with a PDE:
\begin{theorem}\label{new-continuation}
	Starting from $\bm{\theta}_0\sim \pi_0$, for fixed mapping $\bm{g}: \mathbb{R}^d\rightarrow \mathbb{R}^d$ and $\forall t\in[0,\tau_k]$, let
	\begin{equation}	
	\bm{\theta}_t=e^{- \lambda t}\bm{\theta}_0- \frac{1- e^{- \lambda t}}{\lambda}\bm{g}(\bm{\theta_0})+\mathcal{N}\left(0,\frac{1- e^{-2 \lambda t}}{ \beta_k' \lambda}I_d\right).
	\end{equation}
	The pdf $\pi_t$ of $\bm{\theta}_t$ satisfies the following PDE:
	\begin{equation}
	\frac{\partial\pi}{\partial t}=\frac{1}{\beta'_k}\Delta \pi+\bm{\nabla}\cdot(\lambda \pi \bm{w})+\bm{\nabla}\cdot \left(\pi \mathbb{E}\left[ \bm{g}(\bm{\theta}_0)| \bm{\theta}_t=\bm{w}\right]\right)
	\end{equation}
\end{theorem}
Given $\bm{\theta}_0=\bm{y}$ fixed, the conditional density of $\bm{\theta}_t$ is Gaussian pdf. Based on solution to Fokker-Planck Equation for Ornstein-Uhlenbeck process~\citep{risken1989fokker}, we can construct the following equation with $\pi_t$ as its solution:
\begin{equation}
\frac{\partial\pi}{\partial t}=\frac{1}{\beta}\Delta \pi+\bm{\nabla}\cdot(\lambda \pi \bm{w})+\bm{\nabla}\cdot \left(\pi\bm{g}(\bm{y})\right)
\end{equation}
 We then integrate by $\pi_0$ and get the result.

The gradient update in Theorem~\ref{new-continuation} can be related to standard SGLD step by the following transformation of parameters:
\begin{equation}
	\begin{cases}
	\tau_k=&-\frac{1}{\lambda}\ln(1- \eta_k\lambda)\\
	\beta'_k=&\left(1-\frac{\lambda\eta_k}{2}\right)\beta
	\end{cases}
\end{equation}
Using this change of parameters, conditioned on $i_k$, and set $\bm{g}(\bm{w})=\bm{\nabla} f_{i_k}(\bm{w})$, the final distribution $\pi_{\tau_k}$ in Theorem~\ref{new-continuation} is exactly the same with output distribution of SGLD update 
\begin{equation}
	\bm{w}_{k+1}=\bm{w}_k-\eta_k\bm{\nabla} f_{i_k}(\bm{w})+\sqrt{\frac{2\eta_k}{\beta}} \mathcal{N}(0,I_d)
\end{equation}
In Section 3.2, we requires the regularization parameter $\lambda$ to be exactly equal to $\frac{1}{\beta\sigma_0^2}$. However, in the parameter distribution, $\beta_k'$ can vary according to $\eta_k$, making it impossible to fit with fixed parameter $\lambda$. In order to handle this technical issue, we allow the prior distribution to change in a data-independent way during iterations, and let prior at $k$-th round be $\gamma_k$. To exactly cancel out the difference induced by mismatch between regularization parameter and $\beta_k'$, we construct a continuous time prior $\tilde{\gamma}$ satisfying the following PDE:
\begin{equation}
	\frac{\partial\tilde{\gamma}}{\partial t}=\frac{1}{\beta'_k}\Delta \tilde{\gamma}+\bm{\nabla}\cdot(\lambda \tilde{\gamma} \bm{w}),\quad t\in[0,\tau_k]
\end{equation}
It is easy to prove by induction that $\tilde{\gamma}$ is isotropic Gaussian. Let $\tilde{\gamma}_t=\mathcal{N}(0,\tilde{\sigma}_t^2 I_d)$, we have:
\begin{equation}
\tilde{\sigma}_t^2=
	\begin{cases}
	e^{-2\lambda t}\tilde{\sigma}_0^2+\frac{1-e^{-2\lambda t}}{\beta_k' \lambda},&\lambda>0\\
	\tilde{\sigma}_0^2+\frac{t}{\beta_k'},&\lambda=0
	\end{cases}
\end{equation}
Based on the solution above, we have a series of priors $\gamma_k=\mathcal{N}(0,\sigma_k^2 I_d)$, with $\sigma_k^2$ defined via iterative procedure defined above. Putting them together, we are ready to cancel out the $\bm{w}$ term in upper bound for KL divergence.

\subsection{Estimating the KL Divergence}
In this section, we present an upper bound on the KL divergence $D_{KL}(p_k|| \gamma_k)$ based on the interpolation in previous section. We first give the following estimate for one-step SGLD update. In the following, we denote $\mathbb{E}\left[ \bm{\nabla} f_{i_k}(\bm{\theta}_0)| \bm{\theta}_t=\bm{w}\right]$ by $\bm{h}_t(w)$ for convenience.
\begin{lemma}\label{one-step-kl}
	Conditioned on choice of stochastic gradient operator $\bm{g}_k(\cdot)$, consider an SGLD update for regularized ERM with transformed parameters $(\tau_k,\beta_k')$, and let prior $\tilde{\sigma}_t$ be defined above. We have the following inequality:
	\begin{equation}
		D_{KL}\left(p_{k+1}|_{i_k}\Big|\Big|\gamma_{k+1}\right)\leq e^{-\frac{\tau_k}{2b_k}}D_{KL}\left(p_{k}\big|\big|\gamma_{k}\right)+\frac{\beta_k'\tau_k}{2}\mathbb{E} \Vert \bm{g}_k(\bm{w}_k)\Vert^2
	\end{equation}
	where $b_k=\max\left(\tilde{\sigma}_{k-1}^2\beta_k', \frac{1}{\lambda}\right)$ for $\lambda>0$, and $b_k=\tilde{\sigma}_{k-1}^2\beta_k'+\tau_k$ for $\lambda=0$.
\end{lemma}
\begin{proof}
We take derivative of KL divergence between time-varying posterior and time-varying prior.
\begin{equation}
	\begin{split}
	\frac{d}{dt}D_{KL}(\pi_t||\tilde{\gamma}_t)=&\int_{\mathbb{R}^d} \frac{\partial \pi}{\partial t}(\log \pi+1-\log \tilde{\gamma})dw-\int_{\mathbb{R}^d} \frac{\pi}{\tilde{\gamma}}\frac{\partial\tilde{\gamma}}{\partial t}dw \\
	=&\int_{\mathbb{R}^d} \pi\langle \bm{h}_t(\bm{w})+\lambda\bm{w} +\frac{1}{\beta_k'}\bm{\nabla} \log \pi, \bm{\nabla} \log \pi-\bm{\nabla} \log \tilde{\gamma}\rangle dw\\
	&-\int_{\mathbb{R}^d} \pi\langle \lambda\bm{w} +\frac{1}{\beta_k'}\bm{\nabla} \log \tilde{\gamma}, \bm{\nabla} \log \pi-\bm{\nabla} \log \tilde{\gamma}\rangle dw\\
	\leq &-\left(\frac{1}{\beta_k'}-\frac{1}{2C}\right)\int_{\mathbb{R}^d} \pi\Vert \bm{\nabla} \log \pi-\bm{\nabla}\log\tilde{\gamma}\Vert^2dw+\frac{C}{2}\int_{\mathbb{R}^d} \pi \Vert \bm{h}_t\Vert^2dw
	\end{split}
\end{equation}
As in the ideal case, we choose $C=\beta_k'$ and use logarithmic Sobolev inequality for the first term. The variance parameter in the inequality can vary through time. Fortunately, since $\tau_k$ is typically small, we can use worst-case upper bounds for this parameter, which is easy to obtain as $\tilde{\sigma}_t^2$ is monotonic in both cases.
\begin{equation}
\tilde{\sigma}_t^2\leq
	\begin{cases}
		\tilde{\sigma}_0^2+\frac{\tau_k}{\beta_k'},&\lambda=0\\
		\max\left(\tilde{\sigma}_0^2, \frac{1}{\beta_k'\lambda}\right),&\lambda>0
	\end{cases}
\end{equation}
Using the ODE approach in the analysis for ideal case, we can obtain an upper bound for KL divergence after gradient update.
\begin{equation}
		D_{KL}\left(p_{k+1}|_{i_k}\Big|\Big|\gamma_{k+1}\right)\leq e^{-\frac{\tau_k}{2b_k}}D_{KL}\left(p_{k}\big|\big|\gamma_{k}\right)+\frac{\beta_k'\tau_k}{2}\int_{0}^{\tau_k} \int_{\mathbb{R}^d} \pi_t \Vert \bm{h}_t(\bm{w})\Vert^2dwdt
	\end{equation}
For the last integral, we have:
\begin{equation}
	\begin{split}
	\int_{\mathbb{R}^d} \pi_t \Vert \bm{h}_t(\bm{w})\Vert^2dw=&\int_{\mathbb{R}^d} p(\bm{\theta}_t=\bm{w}) \left\Vert \int_{\mathbb{R}^d}\frac{p(\bm{\theta}_t=\bm{w},\bm{\theta}_0=\bm{y})}{p(\bm{\theta}_t=\bm{w})}\bm{g}_k(\bm{y})dy\right\Vert^2dw\\
	\leq& \int_{\mathbb{R}^d} \frac{1}{p(\bm{\theta}_t=\bm{w})} \left(\int_{\mathbb{R}^d}p(\bm{\theta}_t=\bm{w},\bm{\theta}_0=\bm{y})dy\right)\left(\int_{\mathbb{R}^d}p(\bm{\theta}_t=\bm{w},\bm{\theta}_0=\bm{y})\Vert\bm{g}_k(\bm{y})\Vert^2dy\right)dw\\
	=&\mathbb{E} \Vert \bm{g}_k(\bm{w}_k)\Vert^2
	\end{split}
\end{equation}
\end{proof}
Using Lemma~\ref{one-step-kl} iteratively, we can obtain KL divergence upper bounds for the whole SGLD algorithm. Our analysis is divided into 3 cases based on choice of regularization parameter $\lambda$. In the following we always make a mild technical assumption that $\eta_k\lambda<0.5,\forall k$. This makes sure that the transformed parameters are at the same order with original ones, namely, $\frac{3}{4}\beta_k\leq\beta_k'\leq \beta_k$ and $\eta_k\leq\tau_k\leq 2\eta_k$.

\vspace{0.2cm}
\noindent\textbf{Case I: $\lambda=0$.}
\vspace{0.2cm}

In this case, the variance of each prior is $\sigma_k^2=\sigma_{k-1}^2+\frac{\tau_k}{\beta}=\sigma_0+\frac{1}{\beta}\sum_{j=1}^k \tau_j$. So we have $b_k=\sigma_0^2\beta+\sum_{j=1}^k \tau_j\leq \sigma_0^2\beta+2\sum_{j=1}^k \eta_j$. By iteratively using Lemma~\ref{one-step-kl}, we get
\begin{equation}
	D_{KL}(p_N||\gamma_N)\leq \beta\sum_{k=1}^N \eta_k\exp\left( -\sum_{j=k+1}^N\frac{\eta_j}{2\sigma_0^2\beta+4\sum_{l=1}^j \eta_l}\right)\mathbb{E}\left[\Vert\bm{g}_k(\bm{w}_k)\Vert^2\right]
\end{equation}
It is easy to obtain a rough upper bound for above estimate: $
	D_{KL}(p_N||\gamma_N)\leq \beta\sum_{k=1}^N \eta_k\mathbb{E}\left[\Vert\bm{\nabla}f_{i_k}(\bm{w}_k)\Vert^2\right]$, which recovers the bound induced by succinct stability-based analysis, with uniform Lipschitz constant replaced by gradient norm along optimization trajectory.

The multiplicative factor can also make the bound significantly smaller. For example, by choosing $\eta_k=ck^{-\alpha}$ with $\alpha\in [0,1]$, the exponential factor becomes approximately $O\left(\left(k/N\right)^{\frac{1-\alpha}{4}}\right)$ for $k\geq (\beta\sigma_0^2)^{\frac{1}{1-\alpha}}$, which means a polynomially decaying effect for contribution from earlier rounds.

\vspace{0.2cm}
\noindent\textbf{Case II: $0<\lambda\leq \frac{1}{\beta\sigma_0^2}$.}
\vspace{0.2cm}

In this case, we can prove by induction that $\forall k,\sigma_k^2\leq \frac{1}{\lambda\beta}$. So we have $b_k=\frac{1}{\lambda}$ Using Lemma~\ref{one-step-kl} iteratively, we have the following upper bound for KL divergence:
\begin{equation}
	D_{KL}(p_N||\gamma_N)\leq \beta\sum_{k=1}^N \eta_k e^{-\frac{\lambda}{2}(T_N-T_k)}\mathbb{E}\left[\Vert\bm{g}_k(\bm{w}_k)\Vert^2\right]
\end{equation}
where $T_{k}=\sum_{j=1}^k \eta_j$.

\vspace{0.2cm}
\noindent\textbf{Case III: $\lambda> \frac{1}{\beta\sigma_0^2}$.}
\vspace{0.2cm}

In this case, we can prove by induction that $\sigma_k^2\leq e^{-2 \lambda T_k}\sigma_0^2+\frac{4(1-e^{-2\lambda T_k})}{3 \beta \lambda}$. And it is easy to see that $b_k\leq \sigma_{k-1}^2\beta$. For simplicity, we divide the procedure into two parts:
\begin{itemize}
	\item For $T_k\leq \frac{1}{2 \lambda}\ln (\frac{3}{2}\sigma_0^2\beta \lambda)$, we have $\sigma_k^2\leq \frac{4}{3}\sigma_0^2$, and $b_k\leq \frac{4}{3}\sigma_0^2\beta$.
	\item For $T_k> \frac{1}{2 \lambda}\ln (\frac{3}{2}\sigma_0^2\beta \lambda)$, we have $\sigma_k^2\leq \frac{2}{\beta \lambda}$, and $b_k\leq \frac{2}{\lambda}$.
\end{itemize}
Let $k_1\triangleq \min\{k:T_k> \frac{1}{2 \lambda}\ln (\frac{3}{2}\sigma_0^2\beta \lambda)\}$. We can obtain the KL divergence bound by treating two parts differently.
\begin{equation}
	D_{KL}(p_N||\gamma_N)\leq \beta\sum_{k=1}^{k_1} \eta_k e^{-\frac{\lambda}{4}(T_N-T_{k_1})-\frac{3}{8 \beta \sigma_0^2}(T_{k_1}-T_k)}\mathbb{E}\left[\Vert\bm{g}_k(\bm{w}_k)\Vert^2\right]+\beta\sum_{k=k_1+1}^{N} \eta_k e^{-\frac{\lambda}{4}(T_N-T_k)}\mathbb{E}\left[\Vert\bm{g}_k(\bm{w}_k)\Vert^2\right]
\end{equation}
In this case, the contribution of each round will first decay with a slower rate ($\frac{3}{8 \beta \sigma_0^2}$ on the exponent). As variance for each prior becomes smaller along iterations, faster rate of decay with $\frac{\lambda}{4}$ on the exponent will be achieved.

Putting them together, we get the final PAC-Bayesian results:
\begin{theorem}\label{pac-bayes-final}
	Assuming that for $\sigma_k$ defined above, loss function $\ell(w;x)$ is $s_k$-subGaussian with respect to distribution $\mathcal{N}(0,\sigma_k^2I_d)\times \mathcal{D}$. Assume that $f_i(w)$ is uniformly $L$-Lipschitz with respect to $w$. Given algorithmic parameters $N,\{\eta_k\},\beta,\sigma_0,\lambda$ fixed, the following inequalities uniformly holds for SGLD with probability $1-\delta$: (with respect to random draw of training data)
	\begin{equation}
			\mathrm{err}(\bm{w})\leq s_N\left(\frac{\beta}{n}\sum_{k=1}^{N} \eta_k e^{-R_{k,N}}\mathbb{E}\left[\Vert\bm{g}_k(\bm{w}_k)\Vert^2\right]+\frac{\log N/\delta+\log\log NL}{n}\right)^{\frac{1}{2}}
	\end{equation}
	where the decaying factor $R_{k,N}$ is defined as follows:
	\begin{itemize}
		\item If $\lambda=0$, $R_{k,N}=\sum_{j=k+1}^N\frac{\eta_j}{2\sigma_0^2\beta+4T_j}$.
		\item If $0<\lambda\leq \frac{1}{\beta\sigma_0^2}$, $R_{k,N}=\frac{\lambda}{2}(T_N-T_k)$.
		\item If $\lambda> \frac{1}{\beta\sigma_0^2}$,
		$R_{k,N}=
			\begin{cases}
				\frac{\lambda}{4}(T_N-T_{k_1})+\frac{3}{8 \beta \sigma_0^2}(T_{k_1}-T_k),&k<k_1\\
				\frac{\lambda}{4}(T_N-T_{k}),&k\geq k_1
			\end{cases}$
	\end{itemize}
\end{theorem}
Though having a slower $O(1/\sqrt{n})$ rate compared with stability-based bounds, Theorem~\ref{pac-bayes-final} have several advantages which could be helpful for large model classes such as deep learning:
\begin{itemize}
	\item The uniform Lipschitz constant is replaced with norms of actual gradients $\mathbb{E}\Vert\bm{g}_k(\bm{w}_k)\Vert^2$ along optimization trajectory (the expectation is taken only with the randomized algorithm but not with data). The bound has almost no dependence on uniform Lipschitz constant at all. As $L$ usually depends on range of data and even parameters in multi-layer models, it can be large. However, the gradient themselves should not be large, or the optimization trajectories will be unreliable.
	\item The time-decaying factor $e^{-\frac{\lambda}{2}(T_n-T_k)}$ eliminates effect of earlier gradients, which could be much larger than the last few ones. Furthermore, when $\ell^2$ regularization is imposed on a Lipschitz function, the bound should be finite when $T\rightarrow \infty$, as SGLD will not go too far. This phenomenon is properly captured by last two cases in Theorem~\ref{pac-bayes-final}.
\end{itemize}

\section{Conclusion}
In this paper, we study the problem of non-convex (regularized) ERM with Stochastic Gradient Langevin Dynamics, from the perspective of statistical learning theory. Algorithm-dependent generalization bounds are established using uniform stability and PAC-Bayesian theory, respectively. For stability-based results, we get a generalization error bound of $O\left(\frac{k_0+L\sqrt{\beta\sum \eta_i}}{n}\right)$, where $k_0$ is the smallest index $k$ with $\eta_k\beta L^2\leq \ln 2$. This bound attains $O(1/n)$ fast rate and only depends on Lipschitz constant $L$ and aggregated step sizes. For PAC-Bayesian theory with $\frac{\lambda}{2}\Vert w\Vert^2$ regularization, our generalization bound appends a time-decaying factor $R_{k,N}=\frac{\lambda}{2}\sum_{j=k+1}^N \eta_j$ to contribution of each step, and get a generalization bound of $O\left(\sqrt{\frac{\sum \eta_ke^{-R_{k,N}}\mathbb{E}\Vert\bm{g}_k\Vert^2}{n}}\right)$. In addition to time-decaying effect, this bound also depends only on expected norm of gradient taken in optimization trajectory, instead of uniform Lipschitz constant. The bound has no explicit dependence on dimension or norms. This is the first algorithm-dependent generalization bound for non-convex ERM with polynomial dependence on aggregated step sizes and smoothness properties of objective function. Our theoretical results provide potential explanations for generalization performance of deep learning, and emphasizes the merits of Gaussian noise for non-convex learning problems.
\section*{Acknowledgement}
The authors would like to thank Zhou Lu, Feicheng Wang and Xiang Wang for helpful discussions.
\bibliographystyle{chicago}
\bibliography{reference}

\setcounter{section}{0}
\appendix
\renewcommand{\appendixname}{Appendix~\Alph{section}}
\section{Appendix}
\label{appendix}

\subsection{Omitted Proofs in Section 2}
\textbf{Proof of theorem 4}
\begin{proof}
	We use the Donsker-Varadhan change of measure inequality: for any pair of distributions $\mathcal{P}$ and $\mathcal{Q}$ and functional $\phi$, we have
\begin{equation}
	\mathbb{E}_{\mathcal{Q}} (\phi(f))\leq D_{KL}(\mathcal{Q}||\mathcal{P})+\ln\mathbb{E}_{\mathcal{P}} \left(e^{\phi(f)}\right)
\end{equation}
We choose $\phi(f)$ in the form of $\phi(w)=\lambda\left(\mathbb{E}f(w;x)-\hat{\mathbb{E}}_n f(w;x)\right)$, while the values of $\lambda$ will be determined later.

For a finite set $\Lambda\subseteq \mathbb{R}^+$, Markov inequality and union bound guarantee the following with probability at least $1-\delta$ (with respect to randomness of randomly drawn examples):
\begin{equation}
\mathbb{E}_{\mathcal{P}}\left(e^{\lambda\left(\mathbb{E}f(w;x)-\hat{\mathbb{E}}_n f(w;x)\right)}\right)\leq \frac{|\Lambda|}{\delta}\mathbb{E}_{S}\mathbb{E}_{\mathcal{P}}\left(e^{\lambda\left(\mathbb{E}f(w;x)-\hat{\mathbb{E}}_n f(w;x)\right)}\right),\quad \forall \lambda \in \Lambda
\end{equation}
It is easy to see that:
\begin{equation}
\mathbb{E}_{S}\left(e^{\lambda\left(\mathbb{E}f(w;x)-\hat{\mathbb{E}}_n f(w;x)\right)}\right)=\mathbb{E}_{S}\left(e^{\mathbb{E}_{S'}\lambda\left(\hat{\mathbb{E}}'_n f(w;x)-\hat{\mathbb{E}}_nf(w;x)\right)}\right)\leq \mathbb{E}_{S,S'}\left(e^{\lambda\left(\hat{\mathbb{E}}'_n f(w;x)-\hat{\mathbb{E}}_nf(w;x)\right)}\right)
\end{equation}
Given $\lambda$ fixed, we can expand the right hand side based on independence, and control them using the subGaussian property.
\begin{equation}
\mathbb{E}_{S,S',\mathcal{P}}\left(e^{\lambda\left(\hat{\mathbb{E}}'_n f(w;x)-\hat{\mathbb{E}}_nf(w;x)\right)}\right)=\prod_{i=1}^n \mathbb{E}\left(e^{\frac{\lambda}{n}\left( f(w;x_i')-f(w;x_i)\right)}\right)\leq e^{\frac{\lambda^2s^2}{n}}
\end{equation}
Putting everything together, and take $\Lambda=\left\{\frac{1}{s}\sqrt{n\left(2^i+\log\frac{1}{\delta}+\log \log M\right)}\right\}_{i=1}^{\lceil\log M\rceil}$, we have the following uniformly with probability $1-\delta$:
\begin{equation}
\mathbb{E}_{\mathcal{Q}}\left(\mathbb{E}_{\mathcal{D}} f(w;x)- \hat{\mathbb{E}}_n f(w;x)\right)\leq \frac{1}{\lambda}D_{KL}(\mathcal{Q}||\mathcal{P})+\log \frac{|\Lambda|}{\delta}+\frac{\lambda s^2}{n},\quad \forall \lambda\in\Lambda,\mathcal{Q}\in \Xi
\end{equation}
Choosing the index $i$ such that $2^i\leq D_{KL}(\mathcal{Q}||\mathcal{P})<2^{i+1}$, and we get the result.
\end{proof}

\subsection{Omitted Proofs in Section 4}
\textbf{Proof of Theorem 5}
\begin{proof}
	Here we give bound to uniform stability of full gradient SGLD by estimating squared Hellinger distance.

We shall assume $\|\nabla f_i\|\le L$ (which can actually be relaxed to $\|\nabla (f_i-f_j)\|\le 2L$). 

Suppose at step $k$, the starting parameters are $W_{k-1}$ and $W'_{k-1}$ resp. The ending parameters are given by
\begin{equation}
	\bm{w}_{k+1}=\bm{w}_{k}-\frac{\eta_k}{n}\sum\limits_{i=1}^n\bm{\nabla} f_i(\bm{w}_{k})+\sqrt{\frac{2\eta_k}{\beta}}\bm{B}_k
\end{equation}
\begin{equation}
	\bm{w}_{k+1}'=\bm{w}_{k}'- \frac{\eta}{n}\left(\bm{\nabla} f_{i_*}'(\bm{w}'_{k})+\sum\limits_{i=1,i\neq i_*}^n\bm{\nabla} f_i(\bm{w}_{k}')\right)+\sqrt{\frac{2\eta_k}{\beta}}\bm{B}_k'
\end{equation}

where $\bm{B}_k, \bm{B}_k'\sim \mathcal{N}(0, I_d)$.

We consider a family of random variable $\bm{\theta}_t,\bm{\theta}_t'(0\le t\le \eta_k)$ defined by
\begin{equation}
	\bm{\theta}_t=\bm{w}_{k}- \frac{\eta}{n}\sum\limits_{i=1}^n\bm{\nabla} f_i(\bm{w}_{k})+\sqrt{\frac{2t}{\beta}}\bm{B}_k
\end{equation}
\begin{equation}
	\bm{\theta}_t'=\bm{w}_{k}'- \frac{\eta}{n}\sum\limits_{i=1}^n\bm{\nabla} f_i(\bm{w}_{k}')- \frac{t}{n}\Big(\bm{\nabla} f_{i_*}'(\bm{w}_k')-\bm{\nabla}f_{i_*}(\bm{w}_k')\Big)+\sqrt{\frac{2t}{\beta}}\bm{B}_k'
\end{equation}
Till now, we only consider the one-time distribution of $\bm{\theta}_t,\bm{\theta}_t'$, and their dependence on $\bm{w}_k,\bm{w}_k'$, without taking the inter-dependence of whole process into consideration, so we use a simple way of expanding the Gaussian noise. In the actual construction of the SDE, it will be expanded via Brownian motion.

Let the pdf of $\bm{\theta}_t, \bm{\theta}_t'$ be $\pi_t, \pi_t'$. We can see that
\begin{itemize}
	\item $\bm{\theta}_0=\bm{w}_{k}- \frac{\eta}{n}\sum\limits_{i=1}^n\bm{\nabla} f_i(\bm{w}_{k}),\bm{\theta}_0'=\bm{w}_{k}'- \frac{\eta}{n}\sum\limits_{i=1}^n\bm{\nabla} f_i(\bm{w}_{k}')$, so that
	\begin{equation}
		D_H(\pi_0||\pi_0')\le D_H(p_k||p_k')
	\end{equation}
\item the explicit formulae for $\pi_t$ and $\pi_t'$ are given by
\begin{equation}
	\pi_t(\bm{w})=\mathbb{E}_{\bm{w}_k}\left(\frac{\beta}{4 \pi t}\right)^{d/2} \exp\Big(-\frac{\beta}{4t}\Vert\bm{w}-\bm{w}_k+\frac{\eta}{n}\sum\limits_{i=1}^n\bm{\nabla} f_i(\bm{w}_{k})\Vert^2\Big)
\end{equation}
and
\begin{equation}
	\pi_t'(\bm{w})=\mathbb{E}_{\bm{w}_k'}\left(\frac{\beta}{4 \pi t}\right)^{d/2} \exp\Big(-\frac{\beta}{4t}\Vert\bm{w}-\bm{w}_k'+\frac{\eta}{n}\sum\limits_{i=1,i\neq i_*}^n\bm{\nabla} f_i(\bm{w}_{k}')+ \frac{t}{n}\bm{\nabla} f_{i_*}'(\bm{w}_k')\Vert^2\Big)
\end{equation}

Although formidable at first glance, $\pi_t$ and $\pi_t'$ are nothing but superposition of Gaussian density functions w.r.t $\bm{w}$.

Define $\bm{g}_t(\bm{w})$ to be $\bm{0}$ and define $\bm{g}_t'(\bm{w})$ by 
\begin{equation}
\begin{split}
	&\mathbb{E}_{\bm{w}_k'}[\frac{1}{n}\left(\bm{\nabla} f_{i_*}'(\bm{w}_k')-\bm{\nabla}f_{i_*}(\bm{w}_k')\right)|\bm{\theta}_t=\bm{w}]\\
	=& \frac{1}{n\pi_t(\bm{w})}\mathbb{E}_{\bm{w}_k'}\left(\bm{\nabla} f_{i_*}'(\bm{w}_k')-\bm{\nabla}f_{i_*}(\bm{w}_k')\right)\left(\frac{\beta}{4 \pi t}\right)^{d/2} e^{-\frac{\beta}{4t}\Vert\bm{w}-\bm{w}_k'+\frac{\eta}{n}\sum\limits_{i=1,i\neq i_*}^n\bm{\nabla} f_i'(\bm{w}_{k}')+ \frac{t}{n}\bm{\nabla} f_{i_*}(\bm{w}_k')\Vert^2}
\end{split}
\end{equation}

Then by taking derivatives w.r.t to $\bm{w}$ and $t$, we can obtain the following equations, which has the same one-time marginal distribution as $\bm{\theta}_t$ and $\bm{\theta}_t'$ (though they are not the same process):
\begin{equation}
	\frac{\partial \pi_t}{\partial t}=\frac{1}{ \beta}\Delta \pi_t+\bm{\nabla}\cdot\left(\pi_t \bm{g}_t\right)
\end{equation}
\begin{equation}
	\frac{\partial \pi_t'}{\partial t}=\frac{1}{ \beta}\Delta \pi_t'+\bm{\nabla}\cdot\left(\pi_t'\bm{g}_t'\right)
\end{equation}
\end{itemize}

From definition and the assumption $\forall z,z', \|\nabla f(\bm{w};z)-\nabla f(\bm{w};z')\|\le L$, we have
\begin{equation}
	\forall \bm{w},\|\bm{g}_t(\bm{w})-\bm{g}_t'(\bm{w})\|\le \frac{L}{n}
\end{equation}

\begin{align*}
\frac{d}{dt}D_{H}(\pi_t||\pi'_t)&=-\frac{1}{2}\int \frac{1}{ \beta}\sqrt{\pi \pi'}\|\bm{\nabla}\log \frac{\pi}{\pi'}\|^2 + \sqrt{\pi \pi'}\bm{\nabla}\log \frac{\pi}{\pi'}\cdot(\bm{g}_t-\bm{g}_t')\\
&\le \frac{\beta}{8} \int \sqrt{\pi \pi'}\|\bm{g}_t-\bm{g}_t'\|^2\\
&=\frac{\beta L^2}{8n^2}
\end{align*}

As a result, we can estimate the change of squared Hellinger distance in this step:
\begin{align*}
D_H(\pi_{k+1}||\pi_{k+1}')&=D_H(\pi_{\eta_k}||\pi_{\eta_k}')\\
&=D_H(\pi_{0}||\pi_{0}')+\int_0^{\eta_k} \frac{d}{dt}D_{H}(\pi_t||\pi'_t)dt\\
&\le D_H(p_{0}||p_0')+\int_0^{\eta_k} \frac{\beta L^2}{8n^2}\\
&= D_H(p_{0}||p_0')+ \frac{\beta L^2}{8n^2}\eta_k\\
\end{align*}

Then by induction we shall have a final bound for $D_{KL}(\pi||\pi')$ of the form $\frac{\beta L^2}{8n^2}\sum\limits_{k=1}^N \eta_k $.

Then the bound for uniform stability is given by
\begin{equation}
	\epsilon_n\leq O\left(\frac{L\sqrt{\beta\sum_{k=1}^N \eta_k}}{n}\right)
\end{equation}
\end{proof}

\noindent \textbf{Proof of Lemma 1}
\begin{proof}
	For $k=0$, both $p_k$ and $p_k'$ are equal to the prior distribution so that
\begin{equation}
	\int |p_0-p_0'|=0
\end{equation}

Assume the distributions before the $k$ th step is $p_k$ and $p_k'$, and denote the distribution density functions for $\bm{w}_k,\bm{w}_k'$ after $k$ steps conditioned on $i_k=i$ by $p_k^{(i)}, p_k^{(i)\prime}$ respectively, then
\begin{align*}
\int|p_{k+1}-p_{k+1}'|&=\int\left|\frac{1}{n}\sum_{i=1}^n p_k^{(i)}-\frac{1}{n}\sum_{i=1}^n p_k^{(i)\prime}\right|\\
&\le \frac{1}{n}\sum_{i=1}^n \int\left|p_k^{(i)}- p_k^{(i)\prime}\right|\\
\end{align*}

For $i\neq i_*$, $\int|p_k^{(i)}- p_k^{(i)\prime}|\le \int |p_k- p_k'|$ since they undergo the same gradient step and gaussian smoothing.

For $i=i_*$, $\int|p_k^{(i)}- p_k^{(i)\prime}|\le 2$.

As a result, we have
\begin{equation}
	\int|p_{k+1}-p_{k+1}'|\le \int |p_k- p_k'|+ \frac{2}{n}
\end{equation}

By induction, after $k_0$ steps,
\begin{equation}
\label{largeStepSize}
	\int|p_{k_0}-p_{k_0}'|\le \frac{2k_0}{n}
\end{equation}
\end{proof}

\noindent \textbf{Proof of Lemma 2}
\begin{proof}
	Consider the following SGLD update step:
\begin{equation}
 	\bm{w}_{k+1}=\bm{w}_k-\eta_k\nabla f(\bm{w}_k;z_{i_k})+\sqrt{\frac{2\eta_k}{\beta}}\bm{B}_k,\quad \bm{w}_0\sim \mathcal{N}(0,\sigma_0^2I_d),\quad \bm{B}_k\sim \mathcal{N}(0,I_d),\quad i_k\sim\mathcal{U}\{1,2,\cdots,n\}
 \end{equation}
 where $\bm{w}_0,\bm{B}_k,i_k$ are independent. Apparently it is equivalent to the following one:
\begin{align*}
\bm{w}_{k+1}=\bm{w}_k-(1-X)\eta_k\bm{\nabla} f(\bm{w}_k;z_{j_k})-X \eta_k \bm{\nabla}f(\bm{w}_k;z_{i_*})+\sqrt{\frac{2\eta_k}{\beta}}\bm{B}_k,\\\quad \bm{w}_0\sim \mathcal{N}(0,\sigma_0^2I_d),\quad \bm{B}_k\sim \mathcal{N}(0,I_d), \quad j_k\sim\mathcal{U}(\{1,2,\cdots,n\}\setminus \{i_*\})
\end{align*}
where $\bm{w}_0,\bm{B}_k,i_k,X$ are independent and $\mathcal{P}(X=1)=\frac{1}{n}, \mathcal{P}(X=0)=\frac{n-1}{n}$.

As in the case of LMC, we are going to construct a pair of random variable sequences indexed by $t$, and then construct an SDE with the same one-time marginals.

We consider a family of random variables $\bm{\theta}_t~(0\le t\le \eta_k)$ defined by
\begin{equation}
	\bm{\theta}_t=\bm{w}_k-\eta_k\bm{\nabla} f(\bm{w}_k;z_{j_k})-X t (\bm{\nabla}f(\bm{w}_k;z_{i_*})-\bm{\nabla} f(\bm{w}_k;z_{j_k}))+\sqrt{\frac{2t}{\beta}}\bm{B}_k
\end{equation}

Denote pdf of $\bm{\theta}_t$ by $\pi_t$. For neighboring datasets, we also have $\bm{\theta}_t'$ and $\pi_t'$. We can see that
\begin{itemize}
	\item $\bm{\theta}_0=\bm{w}_k-\eta_k\bm{\nabla} f(\bm{w}_k;z_{j_k}),\bm{\theta}_0'=\bm{w}_k'-\eta_k\bm{\nabla} f(\bm{w}_k';z_{j_k})$, so obviously
	\begin{equation}
		D_H(\pi_0||\pi_0')\le D_H(p_k||p_k')
	\end{equation}
	\item $\bm{\theta}_{\eta_k}=\bm{w}_{k+1}$ and $\bm{\theta}_{\eta_k}'=\bm{w}_{k+1}'$
	\item For $0\le t\le \eta_k$, $\pi_t$ and $\pi_t'$ are given by
	\begin{equation}
		\pi_t(\bm{w})=\mathbb{E}_{X,j_k,\bm{w}_k}\left(\frac{\beta}{4\pi t}\right)^{d/2}\exp(-\beta\Vert\bm{w}-\bm{w}_k+\eta_k\bm{\nabla}f_{j_k}(\bm{w}_k)+Xt(\bm{\nabla} f_{i_*}(\bm{w}_k)-\bm{\nabla} f_{j_k}(\bm{w}_k))\Vert^2/(4 t))
	\end{equation}
	and
	\begin{equation}
		\pi_t'(\bm{w})=\mathbb{E}_{X,j_k,\bm{w}'_k}\left(\frac{\beta}{4\pi t}\right)^{d/2}\exp(-\beta\Vert\bm{w}-\bm{w}'_k+\eta_k\bm{\nabla}f_{j_k}(\bm{w}'_k)+Xt(\bm{\nabla} f'_{i_*}(\bm{w}'_k)-\bm{\nabla} f_{j_k}(\bm{w}'_k))\Vert^2/(4 t))
	\end{equation}
\end{itemize}

Although formidable at first glance, $\pi_t$ and $\pi_t'$ are nothing but superposition of Gaussian density functions w.r.t $\bm{w}$. Here $f_{i}(\bm{w}_k)=f(\bm{y};z_{i}),f_{i}'(\bm{y})=f(\bm{y};z'_{i})$.

Define $\hat{\bm{g}}$ by
\begin{equation}
\begin{split}
\label{define_g}
	&\mathbb{E}_{X,j_k,\bm{w}_k}[X(\bm{\nabla} f_{i_*}(\bm{w}_k)-\bm{\nabla} f_{j_k}(\bm{w}_k))|\bm{\theta}_t=\bm{w}]\\
	=&\frac{1}{\pi_t(\bm{w})}\mathbb{E}_{X,j_k,\bm{w}_k} X(\bm{\nabla} f_{i_*}(\bm{w}_k)-\bm{\nabla} f_{j_k}(\bm{w}_k)) \cdot\left(\frac{\beta}{4\pi t}\right)^{d/2}e^{-\beta\Vert\bm{w}-\bm{w}_k+\eta_k\bm{\nabla}f_{j_k}(\bm{w}_k)+Xt(\bm{\nabla} f_{i_*}(\bm{w}_k)-\bm{\nabla} f_{j_k}(\bm{w}_k))\Vert^2/(4 t)}
\end{split}
\end{equation}
and $\hat{\bm{g}}'$ by
\begin{equation}
\begin{split}
\label{define_g'}
	&\mathbb{E}_{X,j_k,\bm{w}_k'}[X(\bm{\nabla} f_{i_*}(\bm{w}_k')-\bm{\nabla} f_{i_*}(\bm{w}_k'))|\bm{\theta}_t'=\bm{w}]\\
	=&\frac{1}{\pi_t(\bm{w}')}\mathbb{E}_{X,j_k,\bm{w}_k'} X(\bm{\nabla} f'_{i_*}(\bm{w}_k')-\bm{\nabla} f_{j_k}(\bm{w}_k')) \cdot\left(\frac{\beta}{4\pi t}\right)^{d/2}e^{-\beta\Vert\bm{w}-\bm{w}_k'+\eta_k\bm{\nabla}f_{j_k}(\bm{w}_k')+Xt(\bm{\nabla} f_{i_*}'(\bm{w}_k')-\bm{\nabla} f_{j_k}(\bm{w}_k'))\Vert^2/(4 t)}
\end{split}
\end{equation}

Then it can be easily verified by calculating derivatives w.r.t $\bm{w}$ and $t$ that:
\begin{equation}
		\frac{\partial \pi}{\partial t}= \frac{1}{ \beta}\triangle \pi+\bm{\nabla}\cdot(\pi \hat{\bm{g}})
	\end{equation}
	and
	\begin{equation}
		\frac{\partial \pi'}{\partial t}= \frac{1}{\beta}\triangle \pi'+\bm{\nabla}\cdot(\pi' \hat{\bm{g}}')
	\end{equation}

With the Lemma \ref{upper bound of average gradient} below and using similar analysis as before, then we compute the time derivative of squared Hellinger distance to be
\begin{align*}
\frac{d}{dt}D_H(\pi_t||\pi_t')&=-\frac{1}{4}\int_{\mathbb{R}^d}\sqrt{\pi \pi'}\left(\frac{1}{\beta}\|\bm{\nabla}\log\frac{\pi'}{\pi}\|^2+\bm{\nabla}\log\frac{\pi}{\pi'}\cdot(\hat{\bm{g}}_t-\hat{\bm{g}}'_t)\right)dw\\
&\le \frac{\beta}{8}\int \sqrt{\pi \pi'}\|\hat{\bm{g}}-\hat{\bm{g}}'\|^2\\
&<\frac{\beta L^2}{n^2}
\end{align*}

So we have
\begin{equation}
	D_H(p_{k+1}||p_{k+1}')= D_H(\pi_{\eta_k}||\pi_{\eta_k}')\le D_H(\pi_0||\pi_0')+\frac{\beta L^2}{n^2}\eta_k\le D_H(p_k||p_k')+\frac{\beta L^2}{n^2}\eta_k
\end{equation}

Then one arrives at the statement by induction.
\end{proof}

\begin{lemma}
\label{upper bound of average gradient}
	With the same assumptions of Lemma 2, there is 
	\begin{equation}
		\int \sqrt{\pi \pi'}\|\bm{g}_t-\bm{g}_t'\|^2\le \frac{4\sqrt{2}L}{(n-1)^2}∫
	\end{equation}
\end{lemma}
\begin{proof}
	Let $u_t,u'_t$ denote the pdfs of $\theta_t, \theta_t'$ conditioned on $X=1$ respectively, and let $v_t,v'_t$ denote the pdfs of $\theta_t, \theta_t'$ conditioned on $X=0$ respectively.

Then it's easily seen from equation \ref{define_g} and equation \ref{define_g'} that
\begin{equation}
	\hat{\bm{g}}_t(\bm{w})= \frac{u_t(\bm{w})}{n\pi_t(\bm{w})} \mathbb{E}(\bm{\nabla}f_{i_*}(\bm{w}_k)-\bm{\nabla}f_{j_k}(\bm{w}_k)|\bm{\theta}_t=\bm{w})
\end{equation}
and
\begin{equation}
	\hat{\bm{g}}'_t(\bm{w})= \frac{u_t'(\bm{w})}{n\pi_t'(\bm{w})} \mathbb{E}(\bm{\nabla}f_{i_*}(\bm{w}_k')-\bm{\nabla}f'_{j_k}(\bm{w}_k')|\bm{\theta}_t'=\bm{w})
\end{equation}

So we have bounds:
\begin{equation}
	\|\hat{\bm{g}}_t(\bm{w})\|\le \frac{u_t(\bm{w})L}{n\pi_t(\bm{w})}
\end{equation}
and
\begin{equation}
	\|\hat{\bm{g}}_t'(\bm{w})\|\le \frac{u_t'(\bm{w})L}{n\pi_t'(\bm{w})}
\end{equation}
Then we have
\begin{align*}
\int_{\mathbb{R}^d} \sqrt{\pi_t \pi_t'}\|\hat{\bm{g}}_t-\hat{\bm{g}}_t\|^2&\le 2\int_ {\mathbb{R}^d}\sqrt{\pi_t \pi_t'}\|\hat{\bm{g}}\|^2 +2\int_{\mathbb{R}^d}\sqrt{\pi_t\pi_t'}\|\hat{\bm{g}}'\|^2\\
&\le 2\sqrt{\int \pi_t\|\hat{\bm{g}}\|^4\int \pi_t'}+2\sqrt{\int \pi_t'\|\hat{\bm{g}}'\|^4\int \pi_t}\\
&=2\sqrt{\int \pi_t\|\hat{\bm{g}}\|^4}+2\sqrt{\int \pi_t'\|\hat{\bm{g}}'\|^4}\\
&\le2\sqrt{\int \pi_t\left(\frac{u_tL}{n \pi_t}\right)^4}+2\sqrt{\int \pi_t'\left(\frac{u_t'L}{n \pi_t'}\right)^4}\\
&\le2L\sqrt{\int \frac{u_t^4}{n((n-1)v_t+ u_t)^3}}+2L\sqrt{\int \frac{u_t^{\prime 4}}{n((n-1)v_t'+u_t')^3}}\\
&\le \frac{2L}{n-1}\sqrt{\int \frac{u_t^4}{v_t^3}}+ \frac{2L}{n-1}\sqrt{\int \frac{u_t^{\prime 4}}{v_t^{\prime 3}}}
\end{align*}

To proceed, we shall first seek to find the PDEs satisfied by $u_t,v_t,u_t',v_t'$.

By definition, the explicit expressions for $u_t,v_t$ are
\begin{equation}
	u_t(\bm{w})=\mathbb{E}_{j_k,\bm{w}_k} \left(\frac{\beta}{4\pi t}\right)^{d/2}\exp(-\beta\Vert\bm{w}-\bm{w}_k+\eta_k\bm{\nabla}f_{j_k}(\bm{w}_k)+t(\bm{\nabla} f_{i_*}(\bm{w}_k)-\bm{\nabla} f_{j_k}(\bm{w}_k))\Vert^2/(4 t))
\end{equation}
and
\begin{equation}
	v_t(\bm{w})=\mathbb{E}_{j_k,\bm{w}_k}\left(\frac{\beta}{4\pi t}\right)^{d/2}\exp(-\beta\Vert\bm{w}-\bm{w}_k+\eta_k\bm{\nabla}f_{j_k}(\bm{w}_k)\Vert^2/(4 t))
\end{equation}

Define $\bm{g}_t(\bm{w})$ by
\begin{equation}
\begin{split}
	&\mathbb{E}_{j_k,\bm{w}_k}\Big[\bm{\nabla} f_{i_*}(\bm{w}_k)-\bm{\nabla} f_{j_k}(\bm{w}_k)\Big|X=1,\bm{\theta}_t=\bm{w}\Big]\\
	=&\frac{1}{u_t(\bm{w})}\mathbb{E}_{j_k,\bm{w}_k} \Big(\bm{\nabla} f_{i_*}(\bm{w}_k)-\bm{\nabla} f_{j_k}(\bm{w}_k)\Big)\cdot\left(\frac{\beta}{4\pi t}\right)^{d/2}e^{-\beta\Vert\bm{w}-\bm{w}_k+\eta_k\bm{\nabla}f_{j_k}(\bm{w}_k)+t(\bm{\nabla} f_{i_*}(\bm{w}_k)-\bm{\nabla} f_{j_k}(\bm{w}_k))\Vert^2/(4 t)}
\end{split}
\end{equation}

Then the following equality holds:
\begin{equation}
	\frac{\partial u_t}{\partial t}= \frac{1}{\beta}\Delta u_t+\bm{\nabla}\cdot(u \bm{g}_t)
\end{equation}

And for $v_t$, the following equality holds:
\begin{equation}
	\frac{\partial v_t}{\partial t}= \frac{1}{\beta}\Delta v_t
\end{equation}

Using the Lemma \ref{8 lemma} below, it follows that for $t\le \eta_k\le \frac{\ln 2}{\beta L^2}$
\begin{equation}
	\int\frac{u_{t}^4}{v_{t}^3}\le 8
\end{equation}

Similarly we have
\begin{equation}
	\int \frac{u_t^{\prime 4}}{v_t^{\prime 3}}\le 8
\end{equation}

As a result,
\begin{equation}
	\int \sqrt{\pi \pi'}\|\bm{g}_t-\bm{g}_t'\|^2\le \frac{4\sqrt{2}L}{(n-1)^2}∫
\end{equation}
\end{proof}

\begin{lemma}
\label{8 lemma}
	Let $u,v\in C^\infty([0,+\infty)\times\mathbb{R}^d)$ satisfying respectively:
	\begin{itemize}
		\item $\frac{\partial u}{\partial t}=	\frac{1}{ \beta}\triangle u+\bm{\nabla}\cdot( u \bm{g}_t)$
		\item$\frac{\partial v}{\partial t}=\frac{1}{ \beta}\triangle v+\bm{\nabla}\cdot( v \bm{g}_t')$
	\end{itemize}
	and $u_0=v_0$.

	Assume that $\|\bm{g}_t-\bm{g}_t'\|\le L$
	
	Then for $t\le \frac{\ln 2}{\beta L^2}$, we have
	\begin{equation}
		\int \frac{u^4_t}{v^3_t}\le 8
	\end{equation}
\end{lemma}
\begin{proof}
	\begin{align*}
	\frac{d}{dt}\int_{\mathbb{R}^d} \frac{u^4_t}{v^3_t}&=\int 4\frac{\partial u}{\partial t}\frac{u^{3}}{v^{3}}-3\frac{\partial v}{\partial t} \frac{u^4}{v^4}\\
	&=\int -4(\frac{1}{\beta}\bm{\nabla} u+ u\bm{g})\cdot\bm{\nabla}\frac{u^{3}}{v^{3}}+3(\frac{1}{\beta}\bm{\nabla} v+ v \bm{g}')\cdot\bm{\nabla}\frac{u^4}{v^4}\\
	&=\int \frac{u^4}{v^{3}}\left\{-4 (\frac{1}{\beta}\bm{\nabla} \log u+ \bm{g})\cdot\bm{\nabla}\log\frac{u^{3}}{v^{3}}+3(\frac{1}{\beta}\bm{\nabla} \log v+  \bm{g}')\cdot\bm{\nabla}\log\frac{u^4}{v^4}\right\}\\
	&=\int \frac{12u^4}{v^{3}}\left\{-(\frac{1}{\beta}\bm{\nabla} \log u+ \bm{g})\cdot\bm{\nabla}\log\frac{u}{v}+(\frac{1}{\beta}\bm{\nabla} \log v+  \bm{g}')\cdot\bm{\nabla}\log\frac{u}{v}\right\}\\
	&=\int \frac{12u^4}{v^{3}}\left\{-\frac{1}{\beta}\|\bm{\nabla}\log\frac{v}{u}\|^2-(\bm{g}-\bm{g}')\cdot\bm{\nabla}\log\frac{u}{v}\right\}\\
	&\le\int \frac{3\beta u^4}{v^{3}} \|\bm{g}-\bm{g}'\|^2\\
	&\le 3 \beta L^2\int_{\mathbb{R}^d} \frac{u^4_t}{v^3_t}
	\end{align*}
	
	Then
	\begin{equation}
		\frac{d}{dt}\ln \int\frac{u^4_t}{v^3_t}\le 3 \beta L^2
	\end{equation}

	For $t\le \frac{\ln 2}{\beta L^2}$, we have
	\begin{equation}
		\ln\int \frac{u^4_t}{v^3_t}\le  \frac{\ln 2}{\beta L^2}\cdot 3 \beta L^2= 3\ln 2
	\end{equation}

	i.e.
	\begin{equation}
		\int \frac{u^4_t}{v^3_t}\le 8
	\end{equation}
\end{proof}
\end{document}